\documentclass{article}

\usepackage[preprint]{icml2026}
\setcitestyle{round}

\usepackage[utf8]{inputenc} 
\usepackage[T1]{fontenc}    
\usepackage{hyperref}       
\usepackage{url}            
\usepackage{booktabs}       
\usepackage{amsfonts}       
\usepackage{microtype}      
\usepackage{xcolor}         

\usepackage{subcaption}
\usepackage{wrapfig}
\usepackage{xspace}
\usepackage{amsmath}
\usepackage{amssymb}
\usepackage{multirow}
\usepackage{algorithm}
\usepackage{algpseudocode}
\usepackage{mathtools}
\usepackage[nice]{nicefrac}
\usepackage[english]{babel}
\usepackage{amsthm}
\usepackage{makecell}
\usepackage[capitalize]{cleveref}
\usepackage{thmtools,thm-restate}

\declaretheorem[name=Theorem,numberwithin=section]{theorem}
\declaretheorem[name=Lemma,sibling=theorem]{lemma}

\addtocontents{toc}{\protect\setcounter{tocdepth}{0}}
\usepackage[textsize=tiny]{todonotes}

\newcommand{\E}[2]{\mathbb{E}_{#1}\left[#2\right]}

\DeclarePairedDelimiter\autobracket{(}{)}
\newcommand{\br}[1]{\autobracket*{#1}}
\DeclarePairedDelimiter\autosquarebracket{[}{]}
\newcommand{\sbr}[1]{\autosquarebracket*{#1}}
\newcommand{\norm}[1]{\left\lVert#1\right\rVert}
\newcommand{\normsq}[1]{\left\lVert#1\right\rVert^2}

\newcommand{\taum}{\tau_{\text{mix}}}

\definecolor{mydarkblue}{rgb}{0,0.08,0.45}
\definecolor{mydarkgreen}{rgb}{0,0.45,0.08}

\newcommand{\red}[1]{\textcolor{red}{#1}}
\newcommand{\blue}[1]{\textcolor{blue}{#1}}
\newcommand{\green}[1]{\textcolor{mydarkgreen}{#1}}

\newcommand{\correspondenceemail}{\detokenize{vaswani.sharan@gmail.com}}

\makeatletter
\DeclareRobustCommand\onedot{\futurelet\@let@token\@onedot}
\def\@onedot{\ifx\@let@token.\else.\null\fi\xspace}

\def\ie{\emph{i.e}\onedot}

\makeatother

\long\def\/*#1*/{}
\icmltitlerunning{Towards Parameter-Free Temporal Difference Learning}

\begin{document}

%

%

\twocolumn[
  \icmltitle{Towards Parameter-Free Temporal Difference Learning}

  \begin{icmlauthorlist}
    \icmlauthor{Yunxiang Li}{ubc}
    \icmlauthor{Mark Schmidt}{ubc}
    \icmlauthor{Reza Babanezhad}{samsung}
    \icmlauthor{Sharan Vaswani}{sfu}
  \end{icmlauthorlist}

  \icmlaffiliation{ubc}{University of British Columbia}
  \icmlaffiliation{samsung}{Samsung AI - Montreal}
  \icmlaffiliation{sfu}{Simon Fraser University}
  \icmlcorrespondingauthor{Sharan Vaswani}{\correspondenceemail}

  \icmlkeywords{Reinforcement Learning, Temporal Difference Learning}

  \vskip 0.3in
]

\printAffiliationsAndNotice{}

\begin{abstract}
Temporal difference (TD) learning is a fundamental algorithm for estimating value functions in reinforcement learning. Recent finite-time analyses of TD with linear function approximation quantify its theoretical convergence rate. However, they often require setting the algorithm parameters using problem-dependent quantities that are difficult to estimate in practice --- such as the minimum eigenvalue of the feature covariance (\(\omega\)) or the mixing time of the underlying Markov chain (\(\taum\)). In addition, some analyses rely on nonstandard and impractical modifications, exacerbating the gap between theory and practice. To address these limitations, we use an exponential step-size schedule with the standard TD(0) algorithm. We analyze the resulting method under two sampling regimes: independent and identically distributed (i.i.d.) sampling from the stationary distribution, and the more practical Markovian sampling along a single trajectory. In the i.i.d.\ setting, the proposed algorithm does not require knowledge of problem-dependent quantities such as \(\omega\), and attains the optimal bias-variance trade-off for the last iterate. In the Markovian setting, we propose a regularized TD(0) algorithm with an exponential step-size schedule. The resulting algorithm achieves a comparable convergence rate to prior works, without requiring projections, iterate averaging, or knowledge of \(\taum\) or \(\omega\).

\end{abstract}

\section{Introduction}
Reinforcement learning (RL) is a general framework for sequential decision making under uncertainty, with successes in robotics~\citep{kober2013reinforcement} and in aligning language models~\citep{uc2023survey}. Value functions underpin value-based algorithms~\citep{sutton2018reinforcement} and are central to actor-critic methods~\citep{konda1999actor}, making efficient policy evaluation a core RL task.

Temporal-difference (TD) learning~\citep{sutton1988learning} is an incremental policy-evaluation method that bootstraps value estimates and scales with linear function approximation. While convergence of TD with linear function approximation~\citep{tsitsiklis1997an,dalal2018finite,lakshminarayanan2018linear,mou2020on,Bhandari2018AFT,patil2023finite,samsonov2024improved} and its variants~\citep{Liu2020TemporalDL,mustafin2024closing} have been analyzed, many analyses require hard-to-estimate problem-dependent parameters or nonstandard modifications (e.g., projections or iterate averaging). We therefore seek to \textit{design a theoretically principled TD algorithm that requires only minimal modifications and does not rely on knowledge of problem-dependent constants.}

To that end, we first consider the independent and identically distributed (i.i.d.) sampling regime, where states are sampled from the stationary distribution of the underlying Markov chain for the evaluated policy. The i.i.d.\ sampling regime is often used as a testbed for designing and analyzing policy-evaluation algorithms~\citep{lakshminarayanan2018linear,dalal2018finite,Bhandari2018AFT,patil2023finite,samsonov2024improved}. 

\citet{dalal2018finite} analyze the TD(0) algorithm using tools from stochastic approximation. Their choice of step-size trades off bias (the rate at which the initialization is forgotten) and variance (from i.i.d.\ sampling) for the last iterate. Follow-up work by \citet{Bhandari2018AFT} relates TD to stochastic gradient descent (SGD) and uses optimization tools to analyze TD. They study three step-size schedules (see \cref{tab:comparison}) under i.i.d.\ sampling. Some require knowledge of the smallest eigenvalue of the state-weighted feature covariance \(\omega\), while others yield slower rates. While these rates hold for the last iterate, they do not achieve the optimal bias--variance trade-off. 
More recent work~\citep{patil2023finite,samsonov2024improved} adopts stochastic approximation, a different technique from ours, and uses tail averaging to achieve the optimal bias--variance trade-off without problem-dependent constants. However, we note that typical practical implementations of TD do not use iterate averaging. \\


\begin{table*}[!h]
    \centering
    \resizebox{\textwidth}{!}{%
    \renewcommand{\arraystretch}{2.0}
    \begin{tabular}{lccccc}
    \hline
    \textbf{Sampling} & \textbf{Step-size} & \textbf{Convergence rate} & 
    \makecell{\textbf{Parameters} \\ \textbf{needed}} & 
    \textbf{Projection} & 
    \makecell{\textbf{Last or Average} \\ \textbf{iterate convergence}}  \\
    \hline
    \multirow{3}{*}{\makecell{i.i.d.}}
        & \makecell{ \(O(t+1)^{-z}, z \in (0, 1)\) \\ \citep{dalal2018finite}}
        & \(O\!\left(\exp\!\left(-\omega T^{1-z}\right) + 1/T^z\right)\)
        & None
        & No
        & Last
        \\[6pt]
    \cline{2-6} 
        & \makecell{\(1/\sqrt{T}\) \\ \citep{Bhandari2018AFT}}
        & \(O\!\left(\tfrac{\sigma^2}{\sqrt{T}}\right)\)
        & None
        & No
        & Average
        \\[6pt]
    \cline{2-6}
        & \makecell{\(O(\omega)\) \\ \citep{Bhandari2018AFT}}
        & \(O\!\left(\exp(-\omega^2 T) + \sigma^2\right)\)
        & \(\omega\)
        & No
        & Last
        \\[6pt]
    \cline{2-6}
        & \makecell{\(O\br{\frac{1}{1 + t \, \omega}}\) \\ \citep{Bhandari2018AFT}}
        & \(O\br{\frac{\sigma^2}{T\omega}}\)
        & \(\omega\)
        & No
        & Last
        \\[6pt]
    \cline{2-6}

    

        & \makecell{ \(O(1)\) \\ \citep{samsonov2024improved}}
        & \(\tilde{O}\br{\exp\br{-\omega T} + \frac{\sigma^2}{\omega^2 \, T}} \)
        & None
        & No 
        & Average
        \\[6pt]
        \cline{2-6}
        
        &\(O\br{\frac{1}{T}^{\nicefrac{t}{T}}}\)  (\textbf{Ours})
        & \(\tilde{O}\br{\exp\br{-\omega T} + \frac{\sigma^2}{\omega^2T}} \)
        & None
        & No
        & Last
        \\[6pt]
    \hline
    \hline
    \multirow{3}{*}{\makecell{Markovian \\samples}}
        & \makecell{\(1/\sqrt{T}\) \\ \citep{Bhandari2018AFT}}
        & \(O\br{\frac{\bigl(1+\taum(1/\sqrt{T})\bigr)}{\omega^2\sqrt{T}}}\)
        & No
        & Yes
        & Average
        \\[6pt]
        \cline{2-6}
        & \makecell{\(O(1/\omega)\) \\ \citep{Bhandari2018AFT}}
        & \(O\br{\exp\br{-2\eta\omega\,T}}
            + O\br{\frac{\eta\,\bigl(1+\taum(\eta)\bigr)}{\,\omega^3}}\)
        & \(\omega\)
        & Yes
        & Last
        \\[6pt]
        \cline{2-6}
        & \makecell{\(O(1/(\omega \, (t+1)))\) \\ \citep{Bhandari2018AFT}}
        & \(O\br{\frac{\bigl(1+\taum(\alpha_T)\bigr)}{\,\omega^3}\cdot\frac{1+\log T}{T}}\)
        & \(\omega\)
        & Yes
        & Average
        \\[6pt]


        \cline{2-6}

    & \makecell{ \(O(1)\) for TD with data drop \\ \citep{samsonov2024improved}}
        & \(\tilde{O}\br{\exp\br{-\omega T} + \frac{\taum}{\omega^2 \, T}} \)
        & $\taum$
        & No 
        & Average
        \\[6pt]
        \cline{2-6}

        & \makecell{\(O\br{\frac{\omega}{\taum}}\) \\ \citep{mitra2025a}}
        & \(O\br{\exp\br{-\frac{\omega^2(T+1)}{\taum}}} + \tilde{O}\br{\frac{\taum}{\omega^2(T+1)}}\)
        & \(\taum, \omega\)
        & No
        & Average
        \\[10pt]
        \cline{2-6}


        & \makecell{\(O\br{\frac{1}{\ln{T}}\frac{1}{T}^{\nicefrac{t}{T}}}\) \\ (\textbf{Ours})}
        & \(O\br{\exp\br{-\frac{\omega T}{\ln^3(T)}} + \frac{\ln^4 (T)}{\omega^2T}\exp\br{\frac{m}{\ln (1/\rho)}}} \)
        & $\omega$
        & No
        & Last
        \\[6pt]
        \cline{2-6}

        & \makecell{\(O\br{\frac{1}{\sqrt{T} \, \ln{T}}\frac{1}{T}^{\nicefrac{t}{T}}}\) for regularized TD \\ (\textbf{Ours})}
        & \(O\br{\exp\br{-\frac{\omega \sqrt{T}}{\ln^3(T)}} + \frac{\ln^3 (T)}{\omega^2T}\exp\br{\frac{m}{\ln (1/\rho)}}} \)
        & None
        & No
        & Last
        \\[6pt]
        \hline
    \end{tabular}
    }
    \\[6pt]
    \caption{Comparison of our method and other methods. \(\omega\) is the smallest eigenvalue of the state feature covariance matrix, \(T\) is the number of updates, \(\taum\) is the mixing time defined in~\cref{eq:reg_taum_def}, and \(m,\rho\) are constants in~\cref{def:reg_mixing_time},
    Our i.i.d. result and regularized TD(0) under Markovian sampling require no projections, no prior knowledge of \(\taum\) or \(\omega\), and no iterate averaging.
    }
    \label{tab:comparison}
\end{table*}

\textbf{Contribution 1}. For TD(0) with linear function approximation under i.i.d. sampling, we take an optimization lens similar to \citet{Bhandari2018AFT} and develop a TD algorithm that uses exponentially decaying step-sizes~\citep{li2021second}. Such exponential decaying step-sizes have been used with SGD for minimizing smooth, strongly convex objectives~\citep{vaswani2022towards}. Although the TD(0) update shares certain properties with SGD, it is not the gradient of a fixed objective. Nevertheless, we are the first to prove that TD(0) with exponentially decaying step-sizes achieves the optimal bias-variance trade-off for the last iterate, and does not require knowledge of problem-dependent constants such as \(\omega\) (\cref{sec:exp_iid}).

Since obtaining direct access to the stationary distribution is unrealistic, the i.i.d.\ regime is impractical. Consequently, many theoretical works analyze TD(0) with Markovian sampling~\citep{Bhandari2018AFT,samsonov2024improved,patil2023finite,mou2020on,chandak2025concentration}. In this setting, data are collected along a single Markovian trajectory, introducing temporal dependence that complicates analysis. To enable analysis, prior work often assumes fast mixing so the state distribution approaches stationarity exponentially quickly. Under this assumption, \citet{Bhandari2018AFT} analyze a projected variant of TD(0) under three step-size schedules. Similar to the i.i.d.\ case, the algorithm does not achieve the optimal trade-off between bias and dependence on the mixing time. Moreover, the projection step is nonstandard in practice and requires knowledge of \(\omega\). \citet{chandak2025concentration} treat TD(0) as a contractive stochastic approximation algorithm. However, they require \(\omega\) to set the initial step-size and guarantee convergence only to a neighborhood of the solution.

Other recent Markovian analyses fall into two categories: (i) \citet{srikant2019finite,mitra2025a}, which control correlations between consecutive samples and prove convergence without projection; and (ii) \citet{samsonov2024improved,patil2023finite}, which study TD with data drop, a nonstandard variant that does not explicitly analyze consecutive-sample correlations. Both approaches can achieve the optimal trade-off between bias and mixing-time dependence, but they require knowledge of the mixing time (which is difficult to estimate) and prove only average-iterate convergence. Furthermore, algorithms that discard samples~\citep{samsonov2024improved,patil2023finite} are sample-inefficient, and unlikely to be used in practice. 

\textbf{Contribution 2}. In the Markovian sampling regime, we show that standard TD(0) with linear function approximation and exponentially decaying step-sizes achieves the optimal bias--mixing time trade-off. The algorithm requires neither projections, iterate averaging, nor data drop, and it does not require the mixing time (\cref{sec:standard_td0}); however, it still depends on \(\omega\). We remove this dependence by analyzing a regularized TD(0) variant~\citep{patil2023finite} with exponentially decaying step-sizes. Unlike \citet{patil2023finite}, who use regularization to improve constants, we use it to make the algorithm parameter-free from problem-dependent constants. Our result remains parameter-free while retaining the benefits of standard TD(0) (\cref{sec:regularized_td0}).



Table~\ref{tab:comparison} compares our results with prior work by convergence rate, required parameters, projection, and whether average- or last-iterate convergence is guaranteed. The rest of the paper is organized as follows: \cref{sec:notations} formalizes the problem and notation and introduces TD(0) with exponentially decaying step-sizes. \cref{sec:exp_iid} presents the i.i.d. analysis. \cref{sec:exp_markov} extends the analysis to Markovian sampling. \cref{sec:regularized_td0} establishes our parameter-free regularized TD(0) guarantees.

\section{Problem Formulation}\label{sec:notations}
In this section, we formalize the setting and notation, including the Markov decision process (MDP) and TD(0), linear value-function approximation and assumptions, and an exponential step-size schedule.
\vspace{-2ex}
\paragraph{Markov decision process.} We consider a discounted MDP $M = (\mathcal{S}, \mathcal{A}, \pi, P_\pi, \mu_0, r, \gamma)$, where $\mathcal{S}$ is the state space, $\mathcal{A}$ is the action space, $\pi$ is a fixed policy mapping each state $s \in \mathcal{S}$ to a distribution over actions in $\mathcal{A}$, $P_\pi \in \mathbb{R}^{|\mathcal{S}|\times |\mathcal{S}|}$ is the transition matrix induced by $\pi$ with entries $(P_\pi)_{ij} \triangleq \mathbb{P}(s_{t+1} = s_j \mid s_t = s_i)$, $\mu_0$ is the initial state distribution, $r: \mathcal{S} \times \mathcal{A} \rightarrow \mathbb{R}$ is the reward function, and $\gamma \in (0, 1]$ is the discount factor.
At time $t$, an action $a_t \sim \pi(\cdot \mid s_t)$ is selected, the next state $s_{t+1} \sim P_\pi(\cdot \mid s_t)$ is sampled, and a reward $r(s_t, a_t)$ is received. Because the policy $\pi$ is fixed, we define the expected immediate reward as \(r(s) \triangleq \mathbb{E}_{a \sim \pi(\cdot \mid s)}[r(s, a)]\). For simplicity, we assume $r(s) \in [0, 1]$.
Iterating this interaction produces a trajectory $\tau = (s_0, a_0, s_1, \dots)$ with distribution $p_\pi(\tau)$ under policy $\pi$. Let $\mu_\pi$ denote the stationary state distribution induced by $\pi$. The initial state distribution $\mu_0$ may differ from $\mu_\pi$. The state distribution at time $t$ is $P_\pi^t \mu_0$.
The value function $V^\pi$ gives the expected cumulative discounted reward starting from $\mu_0$ and following policy $\pi$, i.e., \(V^\pi(s) = \mathbb{E}_{\tau \sim p_\pi(\cdot \mid \mu_0)}\left[ \sum_{t=0}^\infty \gamma^t r(s_t) \right]\).

We consider three regimes for sampling: mean-path, i.i.d., and Markovian. In the mean-path setting, expectations are evaluated exactly under \(\mu_\pi\). In the i.i.d. setting, samples are independent draws from \(\mu_\pi\). In the Markovian setting, the process starts from \(\mu_0\) and evolves along a single trajectory.

With these sampling regimes in place, we now turn to policy evaluation for \(\pi\) using TD methods. Given a sampled transition $(s_t,a_t,s_t^\prime)$ at iteration $t$, TD with one-step bootstrapping (referred to as TD(0)) updates the value estimate as follows:
\begin{align*}
V(s_t) \leftarrow V(s_t) + \eta_t \left(r(s_t) + \gamma V(s_t^\prime) - V(s_t)\right)\tag{TD(0) update} 
\end{align*}
where $\eta_t > 0$ is the step-size at time $t$. For simplicity, we omit the superscript $\pi$ in $V^\pi$.


\paragraph{Linear value function approximation.} The above TD(0) update is done on a per-state basis, and becomes computationally expensive for MDPs with a large state space. Consequently, previous works \citep{tsitsiklis1997an,Bhandari2018AFT} consider a linear approximation of the value function. Specifically, for parameters $w \in \mathbb{R}^d$ to be estimated, these works assume that $V_w(s) = w^\top \phi(s)$, where $\phi(s) \in \mathbb{R}^d$ is the known feature vector of state $s$. Given a sampled transition $(s_t, a_t, s_t^\prime)$, the linear TD(0) update \citep{sutton1988learning} is given by:
\begin{align}
w_{t+1} 
&= w_t + \eta_t \left(r(s_t) + \gamma w_t^\top \phi(s_t^\prime) - w_t^\top \phi(s_t)\right) \phi(s_t) \nonumber\\
&= w_t + \eta_t g_t(w_t)\,. \label{eq:td_update} 
\end{align}
where \(g_t(w_t) := \left(r(s_t) + \gamma w_t^\top \phi(s_t^\prime) - w_t^\top \phi(s_t)\right) \phi(s_t)\) is the TD(0) direction at iteration $t$ and $\eta_t$ is the corresponding step-size. 
It is convenient to define the expected TD(0) update direction $g(w)$ where the expectation is over the stationary distribution $\mu_\pi$. In particular, \(g(w) := \mathbb{E}_{s\sim \mu_\pi, s^\prime\sim P(\cdot\mid s)}\left[ \phi(s) \left( r(s) + (\gamma \phi(s^\prime) - \phi(s))^\top w \right) \right]\), referred to as the \textit{mean-path update}.
To analyze convergence, we adopt the following standard MDP assumptions:
\begin{restatable}{assumption}{StationaryDisAssump} \label{ass:stationary_dis}
The Markov chain induced by policy $\pi$ is irreducible and aperiodic, and there exists a unique stationary distribution $\mu_\pi$.
\end{restatable}

The next assumption concerns the feature vectors used in the linear approximation. Let $n$ denote the number of states. Define $\Phi \in \mathbb{R}^{n \times d}$ as the feature matrix whose $i$-th row is $\phi(s_i)^\top$, and $D := \mathrm{diag}(\mu_\pi(s_1), \ldots, \mu_\pi(s_n))$ as the diagonal matrix of stationary state probabilities. Finally, let $\Sigma := \Phi^\top D \Phi$, and let $\omega$ denote its smallest eigenvalue.

\begin{restatable}{assumption}{PhiAssump}  \label{ass:phi}
The feature matrix \(\Phi = [\phi(s_1)^\top, \dots, \phi(s_n)^\top] \in \mathbb{R}^{n \times d}\) has full column rank, which ensures a unique solution \(w^*\). In addition, \(\|\phi(s)\|^2 \leq 1\) for all \(s\).
\end{restatable}
Under \cref{ass:stationary_dis} and \cref{ass:phi}, TD(0) with suitable step-size \(\eta_t\) converges to the unique fixed point $w^*$ with \(g(w^*) = 0\) \citep{Bhandari2018AFT}. We next study how to choose \(\eta_t\) for the TD(0) update in~\cref{eq:td_update}. 

\paragraph{Exponential step-size schedule.}
We adopt an exponential schedule \citep{li2021second}. For a fixed number of iterations $T$, set \(\eta_t = \eta_0 \, \alpha^t\) with \(\alpha = (\nicefrac{1}{T})^{\nicefrac{1}{T}}\). This schedule is effective for smooth, strongly convex problems and adapts to noise without prior knowledge of its level.

\section{Exponential step-size with i.i.d. sampling} \label{sec:exp_iid}
Prior analyses under i.i.d.\ sampling either set step sizes based on problem-dependent constants \citep{Bhandari2018AFT,mustafin2024closing} or provide guarantees only for an averaged iterate \citep{Bhandari2018AFT,samsonov2024improved}, limiting practical utility. We adopt an exponential step-size schedule and develop a variant of TD(0) for i.i.d.\ sampling that does not require problem-dependent constants and establishes a last-iterate guarantee. We first present optimization-style lemmas that we will use, and then show how the exponential schedule delivers our objective.



We first provide a one-step expansion as follows.
\begin{align*}
\mathbb{E} \left[ \|w_{t+1} - w^*\|^2 \right] & = \|w_t - w^*\|^2 \\
& + \red{2\eta_t \E{s_t\sim\mu_\pi}{ g_t(w_t)^\top (w_t - w^*) } } \\ & + \blue{\eta_t^2 \E{s_t\sim\mu_\pi}{ \|g_t(w_t)\|^2 } }.
\end{align*}
The expectation is taken over i.i.d.\ sampling from the stationary distribution \(\mu_\pi\). We omit the subscript below for brevity.
We next provide lemmas that give upper bounds on the \red{red} and \blue{blue} terms. The \red{red} term can be analyzed using the following lemmas.

\begin{restatable}{lemma}{BoundProdWithV} \label{lemma:bound_prod_with_V}
    [Lemma 3 from \citet{Bhandari2018AFT}] Under the i.i.d.\ sampling, $\forall w, w^*\in \mathbb{R}^d$, 
    \begin{align*}
    \E{}{g_t(w_t)^\top (w_t - w^*)} \ge (1-\gamma) \norm{V_{w} - V_{w^*}}_D^2,         
    \end{align*}
\end{restatable}
This is analogous to a one-point strong monotonicity (or strong convexity) condition in optimization analysis. It lower bounds the alignment between the direction from the optimum to the current iterate \(w^* - w_t\), and the  update direction \(g_t(w)\) by the value-space error \(\norm{V_{w} - V_{w^*}}_D^2\).

\begin{restatable}{lemma}{VwLemma} \label{lemma:V_and_w}
    [Extended Lemma 1 from \citet{Bhandari2018AFT}] Under the i.i.d. sampling, \(\forall w_1, w_2 \in \mathbb{R}^d\), 
    \begin{align*}
    \omega \norm{w_1 - w_2}^2 \le \norm{V_{w_1} - V_{w_2}}_D^2 = \norm{w_1 - w_2}^2_\Sigma 
    \end{align*}     
\end{restatable}
This lemma lower bounds \(\norm{V_{w_1} - V_{w_2}}_D^2\) with \(\omega\norm{w_1 - w_2}^2\), allowing us to use strong-convexity-style arguments. Using these lemmas, we can bound the \red{red} term as follows:
\begin{align*}
2\eta_t \E{}{ g_t(w_t)^\top (w_t - w^*) }  
\le - 2 \eta_t \, (1-\gamma) \, \omega \, \normsq{w_t - w^*}.
\end{align*}

In order to bound the \blue{blue} term, we use the following lemma. 
\begin{restatable}{lemma}{BoundGWithV} \label{lemma:bound_g_with_V}
    [Lemma 5 from \citet{Bhandari2018AFT}] Under the i.i.d. sampling, \(\E{}{\norm{g_t(w)}^2} \le 2\sigma^2 + 8\norm{V_{w} - V_{w^*}}_D^2\), where \(\sigma^2 = \E{}{\norm{g_t(w^*)}^2}\).
\end{restatable}


\(\sigma^2\) is the variance of the TD update at the optimum. This lemma is analogous to the variance control in optimization. 

Combining these bounds on the \red{red} and \blue{blue} terms and setting \(\eta_0 \le \frac{1-\gamma}{8}\), we get the following bound:
\begin{align*}
&\mathbb{E} \left[ \|w_{t+1} - w^*\|^2 \right] \\
\le& \red{\|w_t - w^*\|^2} \left( 1 \red{- 2\eta_0\alpha^t (1-\gamma)\omega} \right) + \blue{2\eta_0^2 \alpha^{2t} \sigma^2}.
\end{align*}

Taking expectation over \(t\in [T]\) and using the fact \((1 - x) \le e^{-x}\), we have 
\begin{align*}
    &\mathbb{E} \left[ \| w_{T} - w^*\|^2 \right] \\
    \le& \|w_0 - w^*\|^2 \exp{\left(-\eta_0\omega(1-\gamma) \sum_{t=1}^T \alpha^t\right)} \\
    &+ 2\sigma^2\eta_0^2 \sum_{t=1}^T\alpha^{2t}\exp{\left( -\eta_0\omega(1-\gamma)\sum_{i=t+1}^T \alpha^i \right)}.
\end{align*}

As shown in \cref{app:helper}, the exponential step-size yields \(\sum_{t=1}^T \alpha^t \ge \frac{\alpha T}{\ln T} - \frac{1}{\ln T}\) and \(\sum_{t=1}^T \alpha^{2t}
\exp\left(-\sum_{i=t+1}^T \alpha^i\right) \le O\br{\frac{\left(\ln(T)\right)^2}{\alpha^2T}}\). The exponential step-size achieves bias-variance trade-off without iterate averaging~\citep{patil2023finite,samsonov2024improved}.
Combining these bounds gives the final rate:
\begin{restatable}{theorem}{ExpIidThm}\label{thm:exp_iid}
    Under Assumption \ref{ass:stationary_dis} and \ref{ass:phi}, TD(0) under i.i.d.\ sampling from the stationary distribution with \(\eta_t = \eta_0 \alpha_t\), where \(\eta_0 = \frac{1-\gamma}{8}\), \(\alpha_t = \alpha^t = \frac{1}{T}^{\nicefrac{t}{T}}\), \(\alpha = \frac{1}{T}^{\nicefrac{1}{T}}\), has the following convergence: 
    \begin{align*}
    &\mathbb{E} \left[ \| w_{T+1} - w^*\|^2 \right] \\
    \le&
    \|w_1 - w^*\|^2 e \exp\left( -\eta_0\omega(1-\gamma)\frac{\alpha T}{\ln T} \right) \\
    &\quad+ \frac{8\sigma^2}{e\left(\omega (1-\gamma)\right)^2}
    \frac{\ln^2 T}{\alpha^2T},
    \end{align*}
    where \(\sigma^2 = \E{}{\norm{g_t(w^*)}^2}\). 
\end{restatable}

The proof of this theorem is in \cref{app:exp_iid}. Our analysis adopts the optimization perspective in \citet{Bhandari2018AFT}. Compared with other methods in~\cref{tab:comparison}, our main contribution here is showing how exponential step-sizes, combined with optimization-based arguments, yield a cleaner proof and improved convergence guarantees for the last iterate. \textit{We emphasize that for TD(0) in the i.i.d. setting, ours is the first work that attains the optimal bias-variance trade-off for the last iterate.}

Compared to the most relevant prior works, we note that ~\citet{Bhandari2018AFT} either obtain a slower \(O(1/\sqrt{T})\) rate or require knowledge of \(\omega\) to obtain faster rates without optimal trade-off between the bias and variance. In contrast, the exponential step-size attains the optimal bias-variance trade-off without requiring difficult-to-estimate problem-dependent quantities. 
\citet{samsonov2024improved,patil2023finite} establish a similar convergence bound attaining the optimal bias-variance trade-off with a universal step-size that does not rely on unknown problem-dependent constants. Compared to~\cref{thm:exp_iid}, both these works rely on the stochastic approximation perspective, and derive the convergence rate only for the averaged iterate. Compared to~\citet{samsonov2024improved,patil2023finite},~\cref{thm:exp_iid} has an additional $O(\ln(T))$ dependence in the variance term. This additional \(\log\) factor comes from \cref{lemma:X-bound,lemma:Y-bound} and also appears in stochastic smooth, strongly convex minimization (see \citet{vaswani2022towards}). Hence, our results provide a complementary bias–variance trade-off: we accept a mild \(\log\) factor in exchange for a meaningful last-iterate guarantee without averaging.  In the future, we will investigate whether these bounds can be sharpened to reduce or remove the extra $\log$ factors.

One limitation of our result is that the dependence on \(\omega\) is quadratic, which matches the best-known results without prior knowledge of \(\omega\). Step-size rules that do not rely on knowledge of \(\omega\) (or a related quantity) to set the step-size incur a quadratic dependence on \(1/\omega\) \citep{patil2023finite,samsonov2024improved}. In contrast, step-sizes with knowledge of \(\omega\) yield linear dependence on \(1/\omega\) \citep{Bhandari2018AFT,samsonov2024improved}. This gap reflects the price of adaptivity to an unknown \(\omega\), rather than a weakness specific to our approach.




\section{Handling Markovian sampling} \label{sec:exp_markov}
We now relax the i.i.d. assumption and consider Markovian sampling, where the TD update uses samples drawn sequentially from a single trajectory of the Markov chain. This setting is more realistic because it does not assume that samples are drawn from the hard-to-estimate stationary distribution, but are instead collected by interacting with the environment. However, since the samples are temporally correlated, the update direction is biased relative to the mean-path update. In particular, when the chain is not at stationarity, in general \(\E{}{g_t(w_t)} \neq g(w_t)\). This requires controlling an additional error term. 

In order to do so, we will use the property that the Markov chain is fast-mixing. This is a standard assumption in the analysis of the TD algorithm~\citep{Bhandari2018AFT,mitra2025a}. In particular, under \cref{ass:stationary_dis}, the \(t\)-step state distribution $\mu_0 P_\pi^t$ started from any \(\mu_0\) converges to the stationary distribution \(\mu_\pi\) geometrically fast, \ie, 
\begin{equation}\label{eq:geo_decay}
    \sup_{\mu_0} d_{\text{TV}}\br{P_\pi^t \mu_0, \mu_\pi} \le m \rho^t, \forall t\in \mathbb{N}_0,
\end{equation}
where \(d_{\text{TV}}\) is the total variation distance, and initial distance \(m\) and mixing speed \(\rho\in(0, 1)\) are positive constants that depend on the underlying Markov chain. This deviation from stationarity is quantified via the \textit{mixing time}.
\begin{restatable}{definition}{RegMixingTimeDef} \label{def:reg_mixing_time}
    Define the mixing time as \(\tau_{\delta}  = \min\{ t\in \mathbb{N}_0 \mid m\rho^t \le \delta \}\), where \(\delta \in (0,1)\).
\end{restatable}
We define $\taum$ as \(\tau_\delta\) for an appropriate $\delta$ to be determined later. 

Using this property of fast-mixing,~\citet{Bhandari2018AFT,mitra2025a} controlled the error term from Markovian sampling. In particular, the analysis in~\citet{Bhandari2018AFT} requires projecting the iterates onto a bounded set containing $w^*$. This projection step is nonstandard in practice, and requires knowledge of $\omega$.~\citet{mitra2025a} avoids this projection step by using an induction argument to show the iterates remain bounded. However, they can only prove convergence for the average iterate (obtained by Polyak--Ruppert averaging) and require knowledge of both $\taum$ and $\omega$. Unlike the analysis in~\citet{mitra2025a}, we show that using exponential step-sizes allows us to prove convergence for the \textit{last iterate} without knowledge of $\taum$. Moreover, to remove dependence on \(\omega\), we use the regularized TD(0) update in~\citet{patil2023finite}. The regularized TD(0) update at iteration $t$ is given by
\begin{align*}
w_{t+1} & = w_t + \eta_t g_t^r(w), \quad \text{where} \\
g_t^r(w) &:= \phi(s_t) \left( r(s_t) + (\gamma \phi(s_{t+1}) - \phi(s_t))^\top w \right) - \lambda w \\
&= g_t(w) - \lambda w,
\end{align*}
and \(\lambda>0\) is the strength of regularization.
The corresponding mean-path regularized direction is defined as
\begin{align*}
g^r(w) & := \E{}{\phi(s_t) \left( r(s_t) + (\gamma \phi(s_{t+1}) - \phi(s_t))^\top w \right)} - \lambda w \\
& = g(w) - \lambda w.    
\end{align*}
We define \(w_r^*\) as the fixed point of the regularized TD(0) update satisfying \(g^r(w_r^*) = 0\). Note that the standard TD(0) update involving \(g_t(w)\) is a special case of \(g_t^r(w)\) with \(\lambda = 0\). We now establish the properties shared by both the regularized and standard TD(0) variants. 

In particular, we first show that the following two lemmas provide upper bounds on \(\norm{g^r_t(w)}\) and \(\norm{g^r(w)}\) for the regularized and standard TD(0) variants.
\begin{restatable}{lemma}{RegMitraSev} \label{lemma:reg_mitra7}
    For stochastic update \(g^r_t\), we have
    \begin{equation*}
        \norm{g^r_t(w)} \le (2+\lambda)\norm{w - w_r^*} + (3+\lambda) \zeta,
    \end{equation*}
    where \(\zeta = \max\{1, \norm{w_r^*}\}\).
    
    For standard TD(0) corresponding to \(\lambda = 0\), \( \norm{g^r_t(w)} \le 2\norm{w - w_r^*} + 3 \zeta \), \(\zeta = \max\{1, \norm{w^*}\}\).
\end{restatable}

\begin{restatable}{lemma}{RegMitraEig} \label{lemma:reg_mitra8}
    For mean-path update \(g^r\), we have
    \begin{equation*}
        \norm{g^r(w)} \le (2+\lambda)\norm{w - w_r^*},
    \end{equation*}
    where \(\zeta = \max\{1, \norm{w_r^*}\}\).
    For standard TD(0) corresponding to \(\lambda = 0\), \( \norm{g^r(w)} \le 2\norm{w - w_r^*} \).
\end{restatable}

For Markovian sampling, we show that the fast-mixing property in~\cref{eq:geo_decay} implies the following result. 
\begin{restatable}{lemma}{RegMixingTimeBound} \label{lemma:reg_mixing_time_bound}
    For any initial state distribution \(\mu_0\), the state distribution at time \(t\) is \(P_\pi^t \mu_0\). For any \(w\), when \(t \geq \tau_{\delta}\),
    \begin{align*}
        \norm{\E{s_t\sim P_\pi^t \mu_0}{g^r_{t}(w)} - g^r(w)} \le 2(2+\lambda) \delta\left(\Vert w \Vert +1 \right).    
    \end{align*}
    For standard TD(0) corresponding to setting \(\lambda = 0\), 
    \begin{align*}
    \norm{\E{s_t\sim P_\pi^t \mu_0}{g_{t}(w)} - g(w)} \le 4 \delta \norm{w} + 1.
    \end{align*}
\end{restatable}
For the initial step-size $\eta_0$ in the TD(0) update, for a fixed $T$, we set $\delta$ and define $\taum$ as follows:
\begin{equation} \label{eq:reg_taum_def}
    \delta = \frac{\eta_0}{2(2+\lambda)T} \quad \text{;} \quad \taum = \tau_\delta. 
\end{equation}
By~\cref{eq:geo_decay}, when $T \ge \frac{\ln\br{2(2+\lambda)T m / \eta_0}}{\ln(1/\rho)}$, this implies that $T \ge \taum := \frac{\ln\br{2(2+\lambda)T m / \eta_0}}{\ln(1/\rho)}$. Furthermore, using~\cref{lemma:reg_mixing_time_bound} implies that the following property holds for all \(t\ge \taum\), 
\begin{equation}\label{eq:reg_taum}
    \begin{aligned}
        \norm{\E{s_t\sim P_\pi^t \mu_0}{g^r_t(w)} - g^r(w)} &\le \eta_0 \, \frac{1}{T} \, \br{\norm{w}+1} \\ 
        & = \eta_T \, \br{\norm{w}+1},
    \end{aligned}
\end{equation}
where \(\eta_0\) and \(\eta_T\) are the step-sizes at iterations \(t=0\) and \(t=T\) in the exponential step-size schedule \(\eta_t = \eta_0 \alpha^t\), respectively. The above equation shows that the deviation between the expected update direction at iteration $t$, \(g_t(w)\), and the corresponding mean-path update \(g(w)\) becomes small once \(t\ge\taum\). 

In the next step, we bound one-step progress similar to the i.i.d. case in~\cref{sec:exp_iid}. In particular, we separate the Markovian component \(g^r_t(w_t) - g^r(w_t)\) from the corresponding mean-path term, and use the following bound. 
\begin{align}
    &\norm{w_{t+1} - w_r^*}^2 \nonumber \\
    & \stackrel{\text{(i)}}{=} \norm{w_t - w_r^*}^2 + 2\eta_t \langle g^r_t(w_t), \,  w_t - w_r^* \rangle + \eta_t^2 \norm{g^r_t(w_t)}^2 \nonumber\\
    &= \norm{w_t - w_r^*}^2 + 2\eta_t \langle g^r_t(w_t) - g^r(w_t), \,  w_t - w_r^* \rangle \nonumber\\
    &\quad + \eta_t^2 \norm{g^r_t(w_t)}^2 + 2\eta_t \langle g^r(w_t), \,  w_t - w_r^* \rangle \nonumber\\
    & \implies \E{s_t\sim P_\pi^t \mu_0}{\norm{w_{t+1} - w_r^*}^2} \nonumber \\
    &\stackrel{\text{(ii)}}{\leq} \green{\norm{w_t - w_r^*}^2 + 2\eta_t \langle g^r(w_t), \,  w_t - w_r^* \rangle + 2\eta_t^2 \norm{g^r(w_t)}^2} \nonumber\\
    &\quad \red{+ 2\eta_t \E{s_t\sim P_\pi^t \mu_0}{\langle g^r_t(w_t) - g^r(w_t), \,  w_t - w_r^* \rangle}} \nonumber\\
    &\quad \blue{+ 2\eta_t^2 \E{s_t\sim P_\pi^t \mu_0}{\norm{g^r_t(w_t) - g^r(w_t)}^2}}\,, \label{eq:reg_separate_markovian_noise}
\end{align}
where (i) uses the TD(0) update and (ii) uses the fact that $\norm{a}^2 \leq 2\norm{a-b}^2 + 2\norm{b}^2$. Note that the two terms involving \(g^r_t(w_t) - g^r(w_t)\) in \red{red} and \blue{blue} capture the Markovian noise, while the remaining terms in \green{green} only depend on mean-path quantities. Unlike the i.i.d. setting, since $g_t$ and $w_t$ are correlated even after conditioning on the randomness at iteration $t$, $\E{s_t\sim P_\pi^t \mu_0}{\langle g^r_t(w_t) - g^r(w_t), \,  w_t - w_r^* \rangle \vert w_t} \neq 0$. Consequently, we follow the proof in~\citet{mitra2025a} and use a strong induction argument to simultaneously control the \red{red} and \blue{blue} terms and show that the iterates remain bounded, \ie, for a constant $B(\taum)$ that depends on the mixing time, $\norm{w_t - w_r^*}^2 \leq B(\taum)$. 

For this, we first note that for $t \geq \taum$,~\cref{eq:reg_taum} allows us to bound the average discrepancy between $g^r_t(w)$ and $g^r(w)$ in terms of $\norm{w}$. Hence, in order to set up the induction we use $t = \taum$ as the base case and first show that \(\norm{w_t - w_r^*}\) is bounded by \(B(\taum)\) for \textit{all} $t \leq \taum$. 
\begin{restatable}{lemma}{RegBBound} \label{lemma:reg_B_bound} 
    For the regularized TD(0) update with exponential step-sizes \(\eta_t = \eta_0\alpha_t\), where \(\eta_0 \leq \frac{1-\gamma}{16\ln(T)}\), \(\alpha_t = \alpha^t = \frac{1}{T}^{\nicefrac{t}{T}}\), \(\alpha = \frac{1}{T}^{\nicefrac{1}{T}}\), if \(T\ge \max\{3, 1/\eta_0\}\), 
    \[
    \forall t \le \taum, \quad \norm{w_t - w^*_r}^2  \leq B(\taum) \quad \textbf{(Base case)}  \,, 
    \]
    where \(B(\taum) := \exp\br{2\br{2 + \lambda}\max\{a,b\}} \,\cdot\, \norm{w_1 - w^*_r + \, \zeta}^2 \), where \(a = \frac{1}{\ln(1/\rho)}\), \(b = \frac{\ln\br{2(2+\lambda) m }}{\ln(1/\rho)}\), \(\zeta = \max\{1, \norm{w_r^*}\}\).
\end{restatable}

With \cref{lemma:reg_B_bound} giving the bound on \(\norm{w_t - w_r^*}^2\) for iterations \(t\le \taum\), the base case for the induction is set up. We now state the inductive hypothesis for iteration $t$. 

\textbf{Inductive Hypothesis}: For a fixed $t$, for all $k \leq t$, $\norm{w_k - w^*_r}^2  \leq B(\taum)$. 

\textbf{Inductive Step}: To complete the induction, we need to show that for a fixed $t$, for all $k \leq t+1$, $\norm{w_k - w^*_r}^2  \leq B(\taum)$. In order to do so, we use the inductive hypothesis and first prove a lemma that controls the size of the update across $\taum$ iterations. Specifically, for $t \geq \taum$, we bound \(w_t\) in terms of \(w_{t-\taum}\) as follows.
\begin{restatable}{lemma}{RegTauBound} \label{lemma:reg_tau_bound}
    Let \(T \ge \max\{3, \frac{1}{\eta_0}\}\), and let \(T\) be large enough that  \(\frac{\ln^2(T)}{T} \leq \frac{1}{2a}\), \(\frac{\ln(T)}{T} \leq \frac{1}{b}\), and \(\ln(T) \ge \max\{a,b\}\). Suppose for all \(t \ge \taum\), if \(\norm{w_k - w_r^*}^2 \le B(\taum)\) for all \(k \in [t]\), then,
    \[
    \norm{w_{t} - w_{t-\taum}}^2 \le c_1^2 \, B(\taum) \, \eta_t^2 \, \ln^4(T),
    \]
    where \(c_1^2 = 2560(2+\lambda)^2\), \(a = \frac{1}{\ln(1/\rho)}\) and \(b = \frac{\ln\br{2(2+\lambda) m }}{\ln(1/\rho)}\).
\end{restatable}

Next, we decompose the first Markovian error term in \red{red} for all $t \geq \taum$. In particular, 
\begin{align*}
\mathbb{E}_t & [\langle g^r_t(w_t) - g^r(w_t), w_t - w_r^*, \rangle] = T_1 + T_2 + T_3 + T_4 \\
\text{s.t} \; T_1 &= \mathbb{E}_t [\langle w_t-w_{t-\taum}, g^r_t(w_t)-g^r(w_t)\rangle],\\
T_2 &= \mathbb{E}_t [\langle w_{t-\taum}-w_r^*, g^r_t(w_{t-\taum})-g^r(w_{t-\taum})\rangle],\\
T_3 &= \mathbb{E}_t [\langle w_{t-\taum}-w_r^*, g^r_t(w_t)- g^r_t(w_{t-\taum})\rangle], \\
T_4 &= \mathbb{E}_t [\langle w_{t-\taum}-w_r^*,g^r(w_{t-\taum})-g^r(w_t)\rangle].
\end{align*}
Note that terms $T_1$, $T_3$ and $T_4$ can be bounded deterministically by using Young's inequality. In particular, $T_1$ can be bounded using~\cref{lemma:reg_tau_bound} and the uniform bound on $\norm{g^r_t(w)}$ and $\norm{g^r(w)}$ given by \cref{lemma:reg_mitra7} and \cref{lemma:reg_mitra8}. Terms $T_3$ and $T_4$ can be bounded by using the Lipschitzness of $g^r$,~\cref{lemma:reg_tau_bound} and using the inductive hypothesis to bound $\norm{w_{t-\taum}-w_r^*}$. For \(T_2\), we use the bound from \cref{eq:reg_taum} after conditioning on \(w_{t-\taum}\). Summing these four terms yields the following lemma: 
\begin{restatable}{lemma}{RegEBound} \label{lemma:reg_e_bound}
    Let \(T \ge \max\{3, \frac{1}{\eta_0}\}\), and let \(T\) be large enough that  \(\frac{\ln^2(T)}{T} \leq \frac{1}{2a}\), \(\frac{\ln(T)}{T} \leq \frac{1}{b}\), and \(\ln(T) \ge \max\{a,b\}\). For \(t\ge \taum\), suppose \(\norm{w_k - w_r^*}^2 \le B(\taum)\) for all \(k \in [t]\). Then:
    \begin{align*}
        &\E{t}{\langle g^r_t(w_t) - g^r(w_t), \, w_t - w_r^* \rangle} \\
        \le& C \, \eta_t \, \ln^2(T) \, B(\taum),
    \end{align*}
    where \(C = C_1 + 3 + 2C_2\), \(C_1 = \frac{c_1}{2}\) and \(C_2 = \frac{c_1 c_2}{2}\), \(c_1 = 2560(2+\lambda)^2\) and \(c_2 = 4 \, \br{2+\lambda}^2 + 4 \, \br{3 + \lambda}^2 + 2 \, \br{2+\lambda}^2\). 
\end{restatable}

Finally, using the inductive hypothesis and the uniform bound on $\norm{g_t^r(w)}$ and $\norm{g^r(w)}$, we can show that the \blue{blue} term in~\cref{eq:reg_separate_markovian_noise} can be deterministically bounded in terms of $B(\taum)$. The following lemma provides this bound.

\begin{restatable}{lemma}{RegGVar} \label{lemma:reg_g_var} 
    Assuming \(\norm{w_k - w_r^*}^2 \le B(\taum), \forall k \in [t]\), then we have
    \begin{align*}
        \E{s_t\sim P_\pi^t \mu_0}{\norm{g^r_t(w_t) - g^r(w_t)}^2} \le C' B(\taum),
    \end{align*}
    where \(C' = 10(3+\lambda)^2\).
\end{restatable}

In order to complete the inductive step, we use~\cref{eq:reg_separate_markovian_noise} and the bounds on the \red{red} and \blue{blue} terms along with the mean-path analysis for the \green{green} term to show that $\normsq{w_{t+1} - w^*_r} \leq B(\taum)$ and therefore for all $k \leq t+1$, $\normsq{w_{k} - w^*_r} \leq B(\taum)$. For this last step, the analysis for the standard and regularized TD(0) updates is different, and we handle them separately. 

\subsection{Standard TD(0)} \label{sec:standard_td0}
We use the above lemmas with \(\lambda=0\). Consequently, $w^*_r = w^*$ and $g^r(w) = g(w)$ where $g(w^*) = 0$. The following lemma shows that if the initial step-size $\eta_0$ is small enough, then we can use the inductive hypothesis to show that $\normsq{w_{t+1} - w^*_r} \leq B(\taum)$. 
\begin{restatable}{lemma}{InductionBound} \label{lemma:induction_bound}
For the standard TD(0) update, when \(T \ge \max\{3, \frac{1}{\eta_0}, \frac{\ln\br{4 T m / \eta_0}}{\ln(1/\rho)}\}\), and \(T\) is large enough that \(\frac{\ln^2(T)}{T} \leq \frac{1}{2a}\), \(\frac{\ln(T)}{T} \leq \frac{1}{b}\), and \(\ln(T) \ge \max\{a,b\}\), for a fixed $t$, if $\normsq{w_{k} - w^*_r} \leq B(\taum)$ for all $k \leq t$ and  
\begin{align*}
    \eta_0 \le \frac{(1-\gamma)\omega}{ 2 \, [C \, \ln^2(T) + C']},         
\end{align*}
then $\norm{w_{t+1} - w^*}^2 \le B(\taum)$, and hence, $\normsq{w_{k} - w^*_r} \leq B(\taum)$ for all $k \leq t+1$.
\end{restatable}
This completes the induction for standard TD(0), and shows that for all $t \in [T]$, $\normsq{w_t - w^*} \leq B(\taum)$. Using~\cref{lemma:reg_e_bound}, this also implies that the \red{red} term in~\cref{eq:reg_separate_markovian_noise} is bounded for all $t \in [T]$. Similarly, the \blue{blue} term is also bounded for all $t \in [T]$ as shown in~\cref{lemma:reg_g_var}. Putting together these results, we state the complete theorem for the standard TD(0) update.

\begin{restatable}{theorem}{ExpMarkovThm} \label{thm:exp_markov}
    The standard TD(0) update with exponential step-sizes \(\eta_t = \eta_0\alpha_t\), where \(\eta_0 = \frac{(1-\gamma)\omega}{ 2 \, [C \, \ln^2(T) + C']}\), \(\alpha_t = \alpha^t = \frac{1}{T}^{\nicefrac{t}{T}}\), and \(T \ge \max\{\frac{1}{\eta_0}, \frac{\ln\br{4 T m / \eta_0}}{\ln(1/\rho)}\}\), and \(T\) is large enough that \(\frac{\ln^2(T)}{T} \leq \frac{1}{2a}\), \(\frac{\ln(T)}{T} \leq \frac{1}{b}\), and \(\ln(T) \ge \max\{a,b\}\), achieves the following convergence rate:
    \begin{align*}
        &\E{}{\norm{w_{T+1} - w^*}^2} \\
        =& O\br{\exp\br{-\frac{\omega^2T}{\ln^3(T)}} + \frac{\ln^4(T)}{\omega^2 T}\exp\br{\frac{m}{\ln(1/\rho)}}},
    \end{align*}
    where \(m\) and \(\rho\) are related to mixing time as \(\taum = \frac{\ln(4Tm/\eta_0)}{\ln(1/\rho)}\).
\end{restatable}
The complete proof can be found in \cref{app:exp_markov}.
Compared with other methods in \cref{tab:comparison}, standard TD(0) with exponential step-size achieves a fast convergence rate without requiring projection onto a bounded set. In addition, our guarantee is for the last iterate. Compared with our i.i.d.\ sampling result in \cref{sec:exp_iid}, the rate under Markovian sampling is comparable. However, it requires a problem-dependent parameter \(\omega\) to set the initial step-size \(\eta_0\). In the next subsection, we will show that the regularized TD(0) update removes the dependence on \(\omega\).

\subsection{Regularized TD(0)} \label{sec:regularized_td0}
In this section, we analyze regularized TD(0).
The following lemma provides a step-size condition that, under the inductive hypothesis, guarantees \(\norm{w_t - w_r^*}^2 \le B(\taum)\) for all \(t \in [T]\).
\begin{restatable}{lemma}{RegInductionBound} \label{lemma:reg_induction_bound}
    For the regularized TD(0) update, when  \(T \ge \max\{3, \frac{1}{\eta_0}, \frac{\ln\br{2(2+\lambda) T m / \eta_0}}{\ln(1/\rho)}\}\), and \(T\) is large enough that \(\frac{\ln^2(T)}{T} \leq \frac{1}{2a}\), \(\frac{\ln(T)}{T} \leq \frac{1}{b}\), and \(\ln(T) \ge \max\{a,b\}\), for a fixed $t$, if $\normsq{w_{k} - w^*_r} \leq B(\taum)$ for all $k \leq t$ and the step-size satisfies 
    \begin{align*}
        \eta_0 \le \frac{\lambda}{[C \, \ln^2(T) + C'] + (8 + 2 \lambda^2)},         
    \end{align*}
    then $\norm{w_{t+1} - w_r^*}^2 \le B(\taum)$, and hence, $\normsq{w_{k} - w^*_r} \leq B(\taum)$ for all $k \leq t+1$.
\end{restatable}

This completes the induction for regularized TD(0).
Compared with the step-size requirement in \cref{lemma:induction_bound}, regularized TD(0) removes the requirement for \(\omega\) and replaces it with \(\lambda\), a regularization parameter that can be set appropriately. 
As in standard TD(0), with~\cref{lemma:reg_e_bound} and~\cref{lemma:reg_g_var}, we obtain upper bounds on the \red{red} and \blue{blue} terms for \(t \in [T]\).
Having bounded all terms in the one-step expansion for all \(t \in [T]\), we combine these bounds and state the final convergence rate, also accounting for the distance between \(w_r^*\) and \(w^*\).

\begin{restatable}{theorem}{RegMarkovThm} \label{thm:reg_markov}
    The regularized TD(0) update with exponential step-size \(\eta_t = \eta_0\alpha_t\), where \(\eta_0 = \frac{\lambda}{[C \, \ln^2(T) + C'] + (8 + 2 \lambda^2)}\), \(\alpha_t = \alpha^t = \frac{1}{T}^{\nicefrac{t}{T}}\), and  \(T \ge \max\{\frac{1}{\eta_0}, \frac{\ln\br{2(2+\lambda) T m / \eta_0}}{\ln(1/\rho)}\}\), and \(T\) is large enough that \(\frac{\ln^2(T)}{T} \leq \frac{1}{2a}\), \(\frac{\ln(T)}{T} \leq \frac{1}{b}\), \(\ln(T) \ge \max\{a,b\}\), and \(\lambda = \nicefrac{1}{\sqrt{T}}\), achieves the following convergence rate:
    \begin{align*}
    &\E{}{\norm{w_{T+1} - w^*}^2} \\
    =& O\br{\exp\br{-\frac{\omega \sqrt{T}}{\ln^3(T)}} + \frac{\ln^4 (T)}{\omega^2T}\exp\br{\frac{m}{\ln(1/\rho)}}},
    \end{align*}
    where \(m\) and \(\rho\) are related to mixing time as \(\taum = \frac{\ln(2(2+\lambda)Tm/\eta_0)}{\ln(1/\rho)}\).
\end{restatable}

\paragraph{Technical novelty compared to~\citet{mitra2025a}.}
Our proof builds on the i.i.d. setting from~\cref{sec:exp_iid} and the induction technique of \citet{mitra2025a}. Compared to~\citet{mitra2025a}, our key technical innovations are threefold. (a) In the proofs of \cref{lemma:reg_B_bound,lemma:reg_tau_bound,lemma:reg_e_bound}, \cref{lemma:reg_opt_distance} through \cref{lemma:Geo_alpha_bound-modified}, we establish key connections between exponential step-sizes and Markovian sampling quantities, in particular, the mixing parameters ($m$, $\rho$). These lemmas allow us to control the terms involving $\tau_{\mathrm{mix}}$ and eliminate the need to set the step-size using $\tau_{\mathrm{mix}}$. In contrast, \citet{mitra2025a} set the step-size according to $\tau_{\mathrm{mix}}$ (see Lemma 2 of \citet{mitra2025a}). (b) \citet{mitra2025a} derive convergence by splitting the analysis before and after the mixing time and selecting step sizes via an impractical comparison, specifically between $\frac{\ln(\lambda)}{0.5\omega(1-\gamma)(T+1)}$ and $\frac{\omega(1-\gamma)}{C\tau_{\mathrm{mix}}}$ (see the proof of Theorem 3). In contrast, our analysis yields last-iterate convergence without such case distinctions or comparisons, and without relying on unknown constants. (c) Our use of regularization to remove the dependence on $\omega$ is novel in this line of work; in comparison, \citet{mitra2025a} require knowledge of $\omega$ to set algorithm parameters (see Theorem 1, Theorem 3, and the proof in their appendix).

Compared to the methods in \cref{tab:comparison}, \citet{Bhandari2018AFT} do not obtain optimal trade-off between the bias and the mixing time, and require projecting onto a ball, which is nonstandard and requires knowledge of \(\omega\). \citet{patil2023finite} analyze the Markovian case with data drop, an impractical algorithm that also requires knowledge of \(\taum\).
Our method requires no projection onto a bounded set or any prior knowledge of problem-dependent parameters. In addition, our guarantee is for the last iterate, which is often more practical than iterate averaging.

Compared to previous works, one limitation of~\cref{thm:reg_markov} is the extra term $\exp\!\left(\frac{m}{\log(1/\rho)}\right)\approx\exp\!\left(\frac{m}{1-\rho}\right)\approx\exp(m\tau)$, which induces an exponential dependence on the mixing time. This dependence is therefore weaker than the linear dependence in prior works. We conjecture that this worse dependence is an artifact of our analysis, and improving this is an important direction for future work.




\section{Conclusion}\label{sec:conclusion}

We address the sensitivity of TD learning to step-size selection and unknown problem parameters by using an exponential schedule \(\eta_t = (1/T)^{t/T}\). Our main contributions are twofold. First, under both i.i.d.\ and Markovian sampling, our method requires no prior knowledge of problem-dependent constants and, in the Markovian case, avoids projections. Second, we prove finite-time, last-iterate convergence guarantees in both settings.
Overall, these results suggest a more practical alternative for TD learning with reduced step-size tuning. For future work, we view high-probability guarantees as an important direction.

\section*{Acknowledgments}
We would like to thank Wenlong Mou, Qiushi Lin and Xingtu Liu for helpful feedback on the paper. This work was partially supported by the Canada CIFAR AI Chair Program, the Natural Sciences and Engineering Research Council of Canada (NSERC) Discovery Grants RGPIN-2022-03669, and enabled in part by support provided by the Digital Research Alliance of Canada (alliancecan.ca).

\bibliographystyle{apalike}
\bibliography{reference}

\clearpage



\clearpage
\appendix
\thispagestyle{empty}
\onecolumn
\section*{Appendix}

\renewcommand{\contentsname}{Contents of Appendix}
\addtocontents{toc}{\protect\setcounter{tocdepth}{3}}
{
  \hypersetup{hidelinks}
  \tableofcontents
}
\clearpage
\thispagestyle{empty}


\section{Additional Inequalities Used in the Proofs}
There are other inequalities regarding \(g(w)\) that we used in the proofs. We list them here for completeness.
\paragraph{\(-g\) is \(2\)-Lipschitz.} \ie \( \|g(w_1) - g(w_2)\| \leq 2 \|w_1-w_2\|\).
\begin{proof}
    \begin{equation} \label{eq:L_lipschitz}
    \begin{aligned}
        \|g(w_1) - g(w_2)\| &= \E{s_t\sim\mu_\pi, s_{t+1}\sim P_\pi\mu_\pi}{\norm{\phi(s_t)(\phi(s_t) - \gamma\phi(s_{t+1}) )^\top(w_1 - w_2)}} \\
        &\le 2 \|w_1 - w_2\|.
    \end{aligned}    
    \end{equation}
\end{proof}

\paragraph{\(-g_t\) is \(2\)-Lipschitz.} \ie $\|g_t(w_1) - g_t(w_2)\| \leq 2 \|w_1-w_2\|$.
\begin{proof}
    \begin{equation} \label{eq:L_i_lipschitz}
\begin{aligned}
    \|g_t(w_1) - g_t(w_2)\| &= \norm{\left[ \phi(s_t)(\phi(s_t) - \gamma\phi(s_{t+1}) )^\top \right](w_1 - w_2)} \\
    &\le 2\|w_1 - w_2\|.
\end{aligned}    
    \end{equation}
\end{proof}

\paragraph{Equations 7 and 8 in \citet{mitra2025a}.}
\begin{equation}\label{eq:mitra7}
\norm{g_t(w)} \le 2\norm{w - w^*} + 4 \zeta,
\end{equation}
where \(\zeta = \max\{1, \norm{w^*}\}\).
Since \(g(w^*)=0\), we have
\begin{equation}\label{eq:mitra8}
\norm{g(w)} \le 2\norm{w - w^*}.
\end{equation}

\section{Proof for Constant step-size with mean-path update}\label{app:constant_mean}
We follow the vanilla TD(0) update proof in \citet{Bhandari2018AFT}.
\begin{restatable}{theorem}{ConstantMeanThm}\label{thm:constant_mean}
For the mean-path TD(0) update with a constant step-size \(\eta \le \frac{1 - \gamma}{8}\), the algorithm achieves the following convergence rate:
\[
\| w_{T} - w^*\|^2 \leq \exp{(- \eta(1-\gamma)\omega T)}\left[ \|w_1 - w^*\|_D^2 \right].
\]
Hence, to obtain accuracy $\|w_T - w^*\| \leq \epsilon$, we need $O\left( \left(\frac{1}{\omega}\right) \log\left(\frac{1}{\epsilon}\right) \right)$ gradient evaluations.
\end{restatable}

\begin{proof}
With the update $w_{t+1} = w_t + \eta g(w_t)$
\[
\|w_{t+1} - w^*\|^2 = \|w_t - w^*\|^2 + 2\eta g(w_t)^\top (w_t - w^*) + \eta^2 \|g(w_t)\|^2.
\]
\begin{align*}
\|w_{t+1} - w^*\|^2
&= \|w_t - w^*\|^2 + 2\eta g(w_t)^\top (w_t - w^*) + \eta^2 \|g(w_t)\|^2 \\
&\leq \|w_t - w^*\|^2 - \left(2\eta (1 - \gamma) - 8\eta^2 \right) \|V_{w^*} - V_{w_t}\|_D^2 \tag{by \cref{lemma:bound_prod_with_V} \cref{lemma:bound_g_with_V}} \\
&\leq \|w_t - w^*\|^2 - \eta(1-\gamma) \|V_{w^*} - V_{w_t}\|_D^2 \tag{$\eta \leq (1-\gamma)/8$} \\
&\leq \|w_t - w^*\|^2 
- \eta(1-\gamma) \omega \|w_t - w^*\|^2. \tag{by \cref{lemma:V_and_w}} \\
& =  (1 - \eta(1-\gamma)\omega)\|w_t - w^*\|^2
\end{align*}
Recursing over \(t\in [T]\), we have
\begin{align*}
    \| w_{T} - w^*\|^2 
    \leq (1 - \eta(1-\gamma)\omega)^T \left[ \|w_1 - w^*\|^2 \right].
\end{align*}
Since $(1 - \eta(1-\gamma)\omega) \leq \exp{(- \eta(1-\gamma)\omega)}$,
\begin{align*}
    \| w_{T} - w^*\|^2 \leq \exp{(- \eta(1-\gamma)\omega T)}\|w_1 - w^*\|^2.
\end{align*}
Then for $\epsilon$ accuracy, the gradient computation is $O(\frac{1}{\omega}\log\left(\frac{1}{\epsilon}\right))$.
\end{proof}

\section{Proof for Constant step-size with i.i.d. sampling}\label{app:constant_iid}
In the following sections, we assume that we have access to i.i.d. observations from the stationary distribution. We follow the vanilla TD(0) update proof in \citet{Bhandari2018AFT}.
\begin{restatable}{theorem}{ConstantIidThm}\label{thm:constant_iid}
If the samples are i.i.d., let the constant step-size satisfy $ \eta \leq (1-\gamma)/8 $, and we have the following convergence rate:
\[
\mathbb{E} \left[ \|w_{T} - w^*\|^2 \right] \leq \exp{(- \eta(1-\gamma)\omega T)} \|w_1 - w^*\|^2 + \eta \frac{2\sigma^2}{(1-\gamma)\omega}.
\]
Then to obtain accuracy $\mathbb{E}\|w_T - w^*\| \leq \epsilon$, we need total gradient computation $O\left( \left(\frac{1}{\epsilon \omega^2}\right) \log\left(\frac{1}{\epsilon}\right) \right)$ with step-size satisfying \(\eta \le \min\left\{\frac{\omega (1 - \gamma)}{8}, \frac{\epsilon (1-\gamma)\omega}{2\sigma^2}\right\}\).
\end{restatable}

\begin{proof}
With the update $w_{t+1} = w_t + \eta g_t(w_t)$
\[
\|w_{t+1} - w^*\|^2 = \|w_t - w^*\|^2 + 2\eta g_t(w_t)^\top (w_t - w^*) + \eta^2 \|g_t(w_t)\|^2.
\]
Taking expectation over the i.i.d.\ samples, we have
\begin{align*}
\mathbb{E} \left[ \|w_{t+1} - w^*\|^2 \right] 
&= \|w_t - w^*\|^2 + 2\eta \E{s_t\sim\mu_\pi}{ g_t(w_t)^\top (w_t - w^*) } + \eta^2 \mathbb{E} \left[ \|g_t(w_t)\|^2 \right] \\
&\leq \|w_t - w^*\|^2 - \left(2\eta (1 - \gamma) - 8\eta^2 \right) \|V_{w^*} - V_{w_t}\|_D^2 + 2\eta^2 \sigma^2 \tag{by \cref{lemma:bound_prod_with_V} \cref{lemma:bound_g_with_V}} \\
&\leq \|w_t - w^*\|^2 - \eta(1-\gamma) \|V_{w^*} - V_{w_t}\|_D^2 + 2\eta^2 \sigma^2 \tag{$\eta \leq (1-\gamma)/8$} \\
&\leq \|w_t - w^*\|^2
- \eta(1-\gamma) \omega \|w_t - w^*\|^2
+ 2\eta^2 \sigma^2. \tag{by \cref{lemma:V_and_w}}
\end{align*}
Taking expectation over \(t\in [T]\), we have
\begin{align*}
    \mathbb{E} \left[ \| w_{T} - w^*\|^2 \right] 
    \leq (1 - \eta(1-\gamma)\omega)^T \left[ \|w_1 - w^*\|^2 \right] 
    + 2 \eta^2 \sigma^2 \sum_{t=0}^\infty (1 - \eta(1-\gamma)\omega)^t
\end{align*}
Since $(1 - \eta(1-\gamma)\omega) \leq \exp{(- \eta(1-\gamma)\omega)}$,
\begin{align*}
    \mathbb{E} \left[ \| w_{T} - w^*\|^2 \right] \leq \exp{(- \eta(1-\gamma)\omega T)}\|w_1 - w^*\|^2 + \eta \frac{2\sigma^2}{(1-\gamma)\omega}.
\end{align*}
Then for $\epsilon$ accuracy, select \(\eta \le \min\left\{\frac{\omega (1 - \gamma)}{8}, \frac{\epsilon (1-\gamma)\omega}{2\sigma^2}\right\}\), the gradient computation is $O\left(\frac{1}{\eta(1 - \gamma)\omega}\log \left( \frac{1}{\epsilon} \right)\right) =O(\frac{1}{\epsilon \omega^2}\log\left(\frac{1}{\epsilon}\right))$.
\end{proof}

\clearpage
\section{Proof for Exponential step-size with i.i.d. sampling} \label{app:exp_iid}
We now complete the proof of \cref{thm:exp_iid}.
\ExpIidThm*
\begin{proof}
With the update $w_{t+1} = w_t + \eta_t g_t(w_t)$
\[
\|w_{t+1} - w^*\|^2 = \|w_t - w^*\|^2 + 2\eta_t g_t(w_t)^\top (w_t - w^*) + \eta_t^2 \|g_t(w_t)\|^2.
\]
Taking expectation over the i.i.d. samples, we have
\begin{align*}
\mathbb{E} \left[ \|w_{t+1} - w^*\|^2 \right] 
&= \|w_t - w^*\|^2 + 2\eta_t \E{s_t\sim\mu_\pi, s_{t+1}\sim P_\pi\mu_\pi}{ g_t(w_t)^\top (w_t - w^*) } + \eta_t^2 \E{s_t\sim\mu_\pi, s_{t+1}\sim P_\pi\mu_\pi}{ \|g_t(w_t)\|^2 } \\
&\leq \|w_t - w^*\|^2 - \left(2\eta_t (1 - \gamma) - 8\eta_t^2 \right) \|V_{w^*} - V_{w_t}\|_D^2 + 2\eta_t^2 \sigma^2 \tag{by \cref{lemma:bound_prod_with_V} \cref{lemma:bound_g_with_V}} \\
&\leq \|w_t - w^*\|^2 - \eta_t(1-\gamma) \|V_{w^*} - V_{w_t}\|_D^2 + 2\eta_t^2 \sigma^2 \tag{$\eta_0 \leq (1-\gamma)/8$} \\
&\leq \|w_t - w^*\|^2
- \eta_t(1-\gamma) \omega \|w_t - w^*\|^2
+ 2\eta_t^2 \sigma^2. \tag{by \cref{lemma:V_and_w}}
\end{align*}
Taking expectation over \(t\in [T]\) and unrolling, we have
\begin{align*}
    \mathbb{E} \left[ \| w_{T} - w^*\|^2 \right] 
    &\leq 
    \|w_0 - w^*\|^2 \prod_{t=1}^T\left( 1 - \eta_0\omega(1-\gamma)\alpha^t \right)
    + 2\sigma^2\eta_0^2 \sum_{t=1}^T\alpha^{2t}\prod_{i=t+1}^T\left( 1-\eta_0\omega(1-\gamma)\alpha^t \right) \\
    &\le \|w_0 - w^*\|^2 \exp{\left(-\eta_0\omega(1-\gamma)  \underbrace{\sum_{t=1}^T \alpha^t}_{X}\right)}
    + 2\sigma^2\eta_0^2 \underbrace{\sum_{t=1}^T\alpha^{2t}\exp{\left( -\eta_0\omega(1-\gamma)\sum_{i=t+1}^T \alpha^i \right)}}_{Y}.
\end{align*}
Applying \cref{lemma:X-bound} to the first term, we have
\begin{align*}
    &\|w_0 - w^*\|^2 \exp{\left(-\eta_0\omega(1-\gamma)\sum_{t=1}^T \alpha^t\right)}  \\
    \le& \|w_0 - w^*\|^2 \exp\br{-\eta_0\omega(1-\gamma) \br{\frac{\alpha T}{\ln(T)} - \frac{1}{\ln(T)}}} \\
    \le& \|w_0 - w^*\|^2 \exp\br{\eta_0\omega(1-\gamma)\frac{1}{\ln(T)}}
    \exp\left( -\eta_0\omega(1-\gamma)\frac{\alpha T}{\ln (T)} \right) \\
    \le& \|w_0 - w^*\|^2 e
    \exp\left( -\eta_0\omega(1-\gamma)\frac{\alpha T}{\ln (T)} \right) \tag{\(\eta_0\le 1, \omega\le 1, (1-\gamma)\le 1, \frac{1}{\ln(T)}\le 1\) when \(T\ge 3\), thus \(\exp\br{\eta_0\omega(1-\gamma)\frac{1}{\ln(T)}} \le e\)}
\end{align*}
And applying \cref{lemma:Y-bound} to the second term, we have
\begin{align*}
    &\sum_{t=1}^T\alpha^{2t}\exp{\left( -\eta_0\omega(1-\gamma)\sum_{i=t+1}^T \alpha^i \right)} \\
    \le& \frac{4\exp\br{\eta_0\omega(1-\gamma)\frac{1}{\ln(T)}}}{e^2   \left(\eta_0 \omega (1-\gamma)\right)^2}
    \frac{\ln^2(T)}{\alpha^2T} \\
    \le& \frac{4e}{e^2   \left(\eta_0 \omega (1-\gamma)\right)^2}
    \frac{\ln^2(T)}{\alpha^2T} \tag{\(\exp\br{\eta_0\omega(1-\gamma)\frac{1}{\ln(T)}} \le e\)} \\
    =& \frac{4}{e   \left(\eta_0 \omega (1-\gamma)\right)^2}
    \frac{\ln^2(T)}{\alpha^2T}
\end{align*}
Putting two terms together, we have the convergence result:
\begin{align*}
    \mathbb{E} \left[ \| w_{T} - w^*\|^2 \right] 
    \le
    \|w_0 - w^*\|^2e\exp\left( -\eta_0\omega(1-\gamma)\frac{\alpha T}{\ln (T)} \right)
    + \frac{8\sigma^2}{e   \left(\omega (1-\gamma)\right)^2}
    \frac{\ln^2(T)}{\alpha^2T}.
\end{align*}
\end{proof}

\section{Proof for Exponential step-size with Markovian sampling} \label{app:exp_markov}
Here we analyze standard TD(0) and its regularized variant with an exponential step-size under Markovian sampling. We include standard TD(0) for reference and focus on regularized TD(0) because it requires no problem-dependent parameters. 
The regularized update at iteration $t$ is
\begin{align*}
    g_t^r(w) &= \phi(s_t) \left( r(s_t) + (\gamma \phi(s_{t+1}) - \phi(s_t))^\top w \right) - \lambda w \\
    &= g_t(w) - \lambda w.
\end{align*}
The corresponding mean-path regularized update is \(g^r(w) = \E{s_t\sim\mu_\pi, s_{t+1}\sim P_\pi\mu_\pi}{\phi(s_t) \left( r(s_t) + (\gamma \phi(s_{t+1}) - \phi(s_t))^\top w \right)} - \lambda w = g(w) - \lambda w\), where \(P_\pi\) is the transition matrix induced by \(\pi\), and \(\mu_\pi\) is the stationary distribution. We define \(w_r^*\) as the fixed point of regularized TD(0) update satisfying \(g^r(w_r^*) = 0\).

It is clear that the standard TD(0) update involving \(g_t(w)\) is a special case of \(g_t^r(w)\) with \(\lambda = 0\). We first state the properties shared by both methods that will be used in the proofs, and then highlight the points where their analyses differ.

\begin{restatable}{lemma}{RegConvexMeanLemma} \label{lemma:reg_convex_mean}
    [Lemma 3 from \citet{Bhandari2018AFT} with regularized update]
    For regularized TD(0) with mean-path sampling, the following inequality holds:
    \begin{equation*}
        \langle g^r(w), w_r^* - w \rangle \ge \sbr{(1-\gamma)\omega+\lambda} \norm{w - w_r^*}^2.
    \end{equation*}
    For standard TD(0) corresponding to \(\lambda = 0\), \( \langle g^r(w), w_r^* - w \rangle \ge (1-\gamma)\omega \norm{w - w_r^*}^2\).
\end{restatable}
\begin{proof}
    Define \(\xi_r = (w_r^* - w)^\top \phi(s)\) and \(\xi_r^\prime = (w_r^* - w)^\top \phi(s_{t+1}) \). Since both $s$ and $s'$ are sampled from the stationary distribution, \(\xi_r\) and \(\xi_r^\prime\) have the same marginal distribution. Using the expression for \(g^r\), 
    \begin{align*}
        g^r(w) 
        &= g^r(w) - g^r(w_r^*) \tag{since \(g^r(w_r^*) = 0\)} \\
        &= g(w) - g(w_r^*) - \lambda (w - w_r^*) \tag{by definition of $g^r$} \\
        &= \E{s_t\sim\mu_\pi}{\phi(s_t)(\gamma \phi(s_{t+1}) - \phi(s_t))(w - w_r^*)} - \lambda (w - w_r^*) \tag{by definition of $g$} \\
        &= \E{s_t\sim\mu_\pi}{\phi(s)(\xi_r - \gamma \xi_r^\prime)} - \lambda (w - w_r^*).
    \end{align*}
    Therefore
    \begin{align*}
        \langle g^r(w), w_r^*-w \rangle 
        &= \E{s_t\sim\mu_\pi, s_{t+1}\sim P_\pi\mu_\pi}{\xi_r (\xi_r - \gamma \xi_r^\prime)} + \lambda \norm{w - w_r^*}^2 \tag{since $\xi_r = \langle \phi(s), w^*_r - w \rangle$} \\
        &= \E{s_t\sim\mu_\pi}{\xi_r^2} - \gamma \E{s_t\sim\mu_\pi, s_{t+1}\sim P_\pi\mu_\pi}{\xi_r \xi_r^\prime} + \lambda \norm{w - w_r^*}^2 \\
        &\ge \E{s_t\sim\mu_\pi}{\xi_r^2} - \gamma \sqrt{\E{s_t\sim\mu_\pi}{\xi_r^2}}\sqrt{\E{s_{t+1}\sim \mu_\pi}{(\xi_r^\prime)^2}} 
            + \lambda \norm{w - w_r^*}^2 
            \tag{by Cauchy-Schwarz}\\
        &= \E{s_t\sim\mu_\pi}{\xi_r^2} - \gamma \E{s_t\sim\mu_\pi}{\xi_r^2}
            + \lambda \norm{w - w_r^*}^2 \tag{since $\xi_r$ and $\xi_r^\prime$ have the same marginal distribution} \\
        &= (1-\gamma)(w_r^* - w)^\top \E{s_t\sim\mu_\pi}{\phi(s_t)\phi(s_t)^\top} (w_r^* - w)
            + \lambda \norm{w - w_r^*}^2 \tag{by the definition of $\xi_r$} \\
        &\ge \br{(1-\gamma)\omega + \lambda}\norm{w - w_r^*}^2. \tag{since $\E{s_t\sim\mu_\pi}{\phi(s_t) \phi(s_t)^\top} \succeq \omega I_d$}
    \end{align*}
\end{proof}

\begin{restatable}{lemma}{RegVarMeanLemma} \label{lemma:reg_var_mean}
    [Lemma 4 from \citet{Bhandari2018AFT} with regularized update]
    For regularized TD(0) with mean-path sampling, 
    \begin{equation*}
        \norm{g^r(w)}^2 \le (8+2\lambda^2) \norm{w - w_r^*}^2.
    \end{equation*}
    For standard TD(0) corresponding to \(\lambda = 0\), \( \norm{g_r(w)}^2 \le 8 \norm{w - w_r^*}^2 \).
\end{restatable}
\begin{proof}
    Similar to the proof of \cref{lemma:reg_convex_mean}, define \(\xi_r = (w_r^* - w)^\top \phi(s_t)\) and \(\xi_r^\prime = (w_r^* - w)^\top \phi(s_{t+1}) \), and note that \(\xi_r\) and \(\xi_r^\prime\) have the same marginal distribution.
    \begin{align*}
        &\norm{g^r(w)}^2 \\
        =& \norm{g^r(w) - g^r(w_r^*)}^2 \tag{since $g^r(w_r^*) = 0$} \\
        =& \norm{g(w) - g(w_r^*) - \lambda\br{w-w_r^*}}^2 \tag{by the definition of $g^r$} \\
        \le& 2\norm{g(w) - g(w_r^*)}^2 + 2\lambda^2\norm{w - w_r^*}^2 \tag{since $(a+b)^2 \leq 2 a^2 + 2 b^2$} \\
        =& 2\norm{\E{s_t\sim\mu_\pi}{\phi(s_t)(\xi_r - \gamma \xi_r^\prime)}}^2 + 2\lambda^2\norm{w - w_r^*}^2 \tag{by the definition of $g$} \\
        \le& 2 \br{ \sqrt{\E{s_t\sim\mu_\pi}{\norm{\phi(s_t)}^2}} \sqrt{\E{s_t\sim\mu_\pi}{ (\xi_r - \gamma \xi_r^\prime)^2 }}}^2 + 2\lambda^2\norm{w - w_r^*}^2 \tag{by Cauchy-Schwarz} \\
        =& 2\E{s_t\sim\mu_\pi}{\norm{\phi(s_t)}^2} \E{s_t\sim\mu_\pi}{ (\xi_r - \gamma \xi_r^\prime)^2 } + 2\lambda^2\norm{w - w_r^*}^2 \\
        \le& 2\br{ 2\E{s_t\sim\mu_\pi}{\xi_r^2} + 2\gamma^2 \E{s_t\sim\mu_\pi}{(\xi_r^\prime)^2}} + 2\lambda^2\norm{w - w_r^*}^2 
            \tag{since \(\norm{\phi(s_t)}^2 \le 1\) and $(a+b)^2 \leq 2 a^2 + 2b^2$}\\
        =& 2 \br{2(1+\gamma^2) (w_r^* - w)^\top \E{s_t\sim\mu_\pi}{\phi(s_t)\phi(s_t)^\top} (w_r^* - w) } 
            + 2\lambda^2\norm{w - w_r^*}^2 
            \tag{since $\xi_r$ and $\xi_r^\prime$ have the same marginal distribution} \\
        \le& (8+2\lambda^2) \norm{w - w_r^*}^2.
            \tag{since \(\norm{\phi}^2 \le 1\) and $\gamma \leq 1$}
    \end{align*}
\end{proof}

\begin{restatable}{lemma}{RegOptDistanceLemma} \label{lemma:reg_opt_distance}
    The distance between the fixed points of regularized and standard TD(0) is bounded by:
    \begin{equation*}
        \norm{w^* - w_r^*} \le \frac{\lambda \norm{w^*}}{\lambda+\omega(1-\gamma)}.
    \end{equation*}
    For standard TD(0) corresponding to \(\lambda = 0\), \( \norm{w^* - w_r^*} = 0 \).
\end{restatable}
\begin{proof}
    By \cref{lemma:bound_prod_with_V},
    \begin{align*}
        (w^* - w)^\top g(w) &\ge (1-\gamma) \norm{V_w - V_{w^*}}_D^2 \\
        &\ge (1-\gamma)\omega \norm{w - w^*}^2 
            \tag{since \(\norm{V_w - V_{w^*}}_D^2 = \norm{\Phi^\top (w - w^*)}_D^2 \ge \omega \norm{w - w^*}^2\)}
    \end{align*}
    Substituting \(w = w_r^*\), we obtain
    \begin{align*}
        (w^* - w_r^*)^\top g(w_r^*) \ge (1-\gamma)\omega \norm{w_r^* - w^*}^2.
    \end{align*}
    By the definition of \(w_r^*\), we have \(g_r(w_r^*) = g(w_r^*) - \lambda w_r^* = 0\), thus \(g(w_r^*) = \lambda w_r^*\). Using this relation with the above inequality,
    \begin{align*}
        &\lambda (w^* - w_r^*)^\top w_r^* \ge (1-\gamma)\omega \norm{w_r^* - w^*}^2 \\
        \Longrightarrow \quad &\lambda (w^* - w_r^*)^\top (w_r^* - w^*) + \lambda (w^* - w_r^*)^\top w^* \ge (1-\gamma)\omega \norm{w_r^* - w^*}^2 \tag{adding/subtracting} \\
        \Longrightarrow \quad &-\lambda \norm{w^* - w_r^*}^2 + \lambda (w^* - w_r^*)^\top w^* \ge (1-\gamma)\omega \norm{w_r^* - w^*}^2 \\
        \Longrightarrow \quad &\sbr{\lambda + \omega(1-\gamma)} \norm{w^* - w_r^*}^2 \le \lambda (w^* - w_r^*)^\top w^* \leq \lambda \norm{w^* - w_r^*}   \norm{w^*} \tag{by Cauchy-Schwarz}\\
        \Longrightarrow \quad &\norm{w^* - w_r^*} \le \frac{\lambda \norm{w^*}}{\lambda+\omega(1-\gamma)}.
    \end{align*}
\end{proof}

The following two lemmas are analogous to Equations 7 and 8 in \citet{mitra2025a}.
\RegMitraSev*
\begin{proof}
    \begin{align*}
    \norm{g^r_t(w)} &= \norm{g^r_t(w) - g^r_t(w_r^*) + g^r_t(w_r^*)} \tag{add/subtract} \\
    &= \norm{g_t(w) - g_t(w_r^*) - \lambda (w - w_r^*) + g^r_t(w_r^*)} \\
    &\le \norm{g_t(w) - g_t(w_r^*) - \lambda (w - w_r^*)} + \norm{g^r_t(w_r^*)} \tag{triangle inequality} \\
    &= \norm{\br{\phi(s_t)(\gamma \phi(s_{t+1}) - \phi(s_t))^\top - \lambda}  (w - w_r^*)}
    + \norm{\phi(s_t) \left( r(s_t) + (\gamma \phi(s_{t+1}) - \phi(s_t))^\top w_r^* \right) - \lambda w_r^*}  \\
    &\le \norm{\br{\phi(s_t)(\gamma \phi(s_{t+1}) - \phi(s_t))^\top }  (w - w_r^*)}
    + \lambda \norm{w - w_r^*} \\
    &\quad + \norm{\phi(s_t)r(s_t)} + \norm{ \phi(s_t)\left( (\gamma \phi(s_{t+1}) - \phi(s_t))^\top w_r^* \right)} 
    + \lambda \norm{w_r^*} 
    \tag{\(\norm{a+b} \le \norm{a} + \norm{b}\)}\\
    &\le \norm{\br{\phi(s_t)(\gamma \phi(s_{t+1}) - \phi(s_t))^\top } }\norm{w - w_r^*}
    + \lambda \norm{w - w_r^*} \\
    &\quad + \norm{\phi(s_t)r(s_t)} + \norm{ \phi(s_t) (\gamma \phi(s_{t+1}) - \phi(s_t))} \norm{w_r^* } 
    + \lambda \norm{w_r^*} 
    \tag{by Cauchy-Schwarz}\\
    &\le (2+\lambda)\norm{w - w_r^*} + (3+\lambda) \zeta,
    \tag{\(r(s_t) \le 1, \norm{\phi(s_t)}\le 1\), \cref{eq:L_i_lipschitz}}
    \end{align*}
    where \(\zeta = \max\{1, \norm{w_r^*}\}\).
\end{proof}
The corresponding bound on the mean-path update is
\RegMitraEig*
\begin{proof}
    \begin{align*}
    \norm{g^r(w)} &= \norm{g^r(w) - g^r(w_r^*)} \tag{\(g^r(w_r^*) = 0\)} \\
    &= \norm{ (g(w) - g(w_r^*)) - \lambda (w - w_r^*)} 
    \tag{by definition of \(g^r(w)\)}\\
    &\le \norm{g(w) - g(w_r^*)} + \lambda \norm{w - w_r^*} 
    \tag{triangle inequality}\\
    &\le (2+\lambda)\norm{w - w_r^*}.
    \tag{by \cref{eq:L_lipschitz}}
    \end{align*}
\end{proof}

In the Markovian case, we also need the definition of mixing time. We restate it below and provide a proof.
\RegMixingTimeDef*
\RegMixingTimeBound*
\begin{proof}
    Let \(\norm{g^r_t(w)}_{\infty}\) denote the supremum of \(\norm{g^r_t(w)}\) over states.
    \begin{align*}
        \norm{g^r_t(w)}_{\infty} 
        &= \max_{s_t\in \mathcal{S}} \norm{g^r_t(w)} \\
        &= \max_{s_t\in \mathcal{S}} \norm{\br{r(s_t) + \gamma w^\top \phi(s_t^\prime) - w^\top \phi(s_t) } \phi(s_t) - \lambda w} 
        \tag{Definition of \(g^r_t(w)\)} \\
        &\le (2+\lambda) \norm{w} + 1 \tag{since $r(s_t) \le 1, \norm{\phi(s_t)}\le 1,   \gamma \leq 1$}.
    \end{align*}
    With \(\tau_{\delta}  = \min\{ t\in \mathbb{N}_0 \mid m\rho^t \le \delta \}\), and $P_\pi^t \mu_0$ representing the probability distribution over the states after $t$ steps with initial state distribution $\mu_0$, we have
    \begin{align*}
        \norm{\E{s_t\sim P_\pi^t \mu_0}{g^r_t(w)} - g^r(w)} 
        &= \norm{\E{s_t\sim P_\pi^t \mu_0}{g^r_t(w)} - \E{s_t\sim \mu_\pi}{g^r_t(w)}}    
        \tag{Definition of \(g^r_t(w)\) and \(g^r(w)\)} \\
        &= \norm{\sum_{s_t\in \mathcal{S}} g^r_t(w) \br{\br{P_\pi^t\mu_0}(s_t) - \mu_\pi(s_t)}} \\
        &\le \sum_{s_t\in \mathcal{S}} \norm{g^r_t(w) \br{\br{P_\pi^t\mu_0}(s_t) - \mu_\pi(s_t)}} 
        \tag{triangle inequality}\\
        & \le \norm{g^r_t(w)}_\infty \sum_{s_t\in \mathcal{S}} | \br{P_\pi^t\mu_0}(s_t) - \mu_\pi(s_t)| 
        \tag{\(\norm{g^r_t(w)} \le \norm{g^r_t(w)}_\infty\)}\\                
        &\le 2 \norm{g^r_t(w)}_{\infty}   \sup_{\mu_0} d_{\text{TV}}\br{P_\pi^t \mu_0, \mu_\pi}
        \tag{by the definition of the total variation distance, and taking the $\sup$ over $s_0$} \\
        &\le 2 \norm{g^r_t(w)}_{\infty}   m\rho^t \tag{using~\cref{eq:geo_decay}}\\
        &\le 2m\rho^t \br{(2+\lambda) \norm{w} + 1} \tag{using the bound on $\norm{g^r_t(w)}_{\infty}$} \\
        &\le 2(2+\lambda)\delta \br{\norm{w} + 1} \tag{using the definition of $t_\delta$}
    \end{align*}
\end{proof}

\begin{restatable}{lemma}{RegTaumBound} \label{lemma:reg_taum_bound} 
    For $\taum$ defined in~\cref{eq:reg_taum_def},
    \[
        \taum = a\ln(T')+b,
    \]
    where \(a\) and \(b\) are constants with \(a = \frac{1}{\ln(1/\rho)}\), \(b = \frac{\ln\br{2(2+\lambda) m}}{\ln(1/\rho)}\) and $T' = \frac{T}{\eta_0}$.
\end{restatable}
\begin{proof}
    By \cref{eq:geo_decay}, \cref{lemma:reg_mixing_time_bound} and the definition of $\taum$ in~\cref{eq:reg_taum_def}, we obtain the expression for \(\taum\):
    \begin{align*}
        \taum &= \frac{\ln\br{2(2+\lambda)T m / \eta_0}}{\ln(1/\rho)} \\
        &= \frac{\ln(T')}{\ln(1/\rho)} + \frac{\ln\br{2(2+\lambda) m}}{\ln(1/\rho)} \tag{defining $T' = \frac{T}{\eta_0}$}
    \end{align*}
\end{proof}

With this expression for \(\taum\), we prove the following lemmas.
\begin{restatable}{lemma}{Geo_alpha_bound} \label{lemma:Geo_alpha_bound} 
    If \(\alpha := \frac{1}{T}^{\nicefrac{1}{T}}\), \(T \ge 3\), \(\eta_0 \le 1\) for \(\taum\) defined in \cref{eq:reg_taum_def}, then for all $t \leq \taum$ 
    \[
        \frac{1-\alpha^{t}}{1-\alpha} \le 4\max\{a, b\} \ln(T'),
    \]
    where \(a = \frac{1}{\ln(1/\rho)}\), \(b = \frac{\ln\br{2(2+\lambda) m}}{\ln(1/\rho)}\) and $T' = \frac{T}{\eta_0}$.
\end{restatable}
\begin{proof}
    Using the expression for \(\taum = a\ln(T')+b\) from \cref{lemma:reg_taum_bound}, 
    \begin{align*}
        \frac{1-\alpha^{t}}{1-\alpha} \leq \frac{1-\alpha^{\taum}}{1-\alpha} 
        &= \frac{1-(1/T)^{\taum/T}}{1 - (1/T)^{1/T}} 
        =\frac{1-(1/T)^{(a\ln(T')+b)/T}}{1 - (1/T)^{1/T}} \tag{since $t \leq \taum$} \\
        &= \frac{1-\exp\br{\frac{-a(\ln(T')   \ln(T))}{T} - \frac{b\ln(T)}{T}}}{1-\exp\br{-\frac{\ln(T)}{T}}}.
    \end{align*}
    Define \(j := \frac{a \ln(T)   \ln(T') + b\ln(T)}{T}\), and \(k := \frac{\ln(T)}{T}\). Note that \(j\) and \(k\) are related by \(j = a\ln (T') k + bk\). We can simplify the above inequality as follows:
    \[
        \frac{1-\alpha^{\taum}}{1-\alpha} = \frac{1-\exp(-j)}{1-\exp(-k)}
    \]
    To bound the numerator and denominator separately, we use the fact that \(\frac{v}{1+v} \le 1-\exp\br{-v} \le v\) for \(v>0\). For the numerator, setting \(v = j\), we have \(1-\exp(-j) \le j\). And for the denominator, setting \(v = k\), we have \(1-\exp(-k) \ge \nicefrac{k}{1+k}\). Combining the above relations, 
    \begin{align*}
        \implies \frac{1-\alpha^{t}}{1-\alpha} 
        &\le \frac{j(1+k)}{k} \nonumber\\
        &= \frac{(a\ln(T')k + bk)(1+k)}{k} \tag{since $j = a\ln(T')k + bk$} \\
        &= (a\ln(T') + b) \br{1 + \frac{\ln(T)}{T}} \tag{since $k = \ln(T) / T$} \\
        &\le (a\ln(T') + b)\br{1+1/e} \tag{\(\frac{\ln(T)}{T}\) decreases after $T = e$, assuming $T \geq 3$} \\
        & \leq 4   \max\{a,b\}   \ln(T'). \tag{since \(\eta_0 \le 1\) and $T \geq 1 \implies T' \geq 1$}
    \end{align*}
\end{proof}

\begin{restatable}{lemma}{Geo_alpha_bound-modified} \label{lemma:Geo_alpha_bound-modified} 
    If \(\alpha := \frac{1}{T}^{\nicefrac{1}{T}}\), \(T \ge \max\left\{3, \frac{1}{\eta_0}\right\}\), and \(T\) is large enough that $\frac{\ln^2(T)}{T} \leq \frac{1}{2a}$ and $\frac{\ln(T)}{T} \leq \frac{1}{b}$, with \(\eta_0 \le 1\), for \(\taum\) defined in \cref{eq:reg_taum_def} and for all $t \leq \taum$, 
    \[
        \frac{\alpha^{-\taum}-1}{1-\alpha} \le 8 \, \max\{a,b\} \, \ln(T')
    \]
    where \(a = \frac{1}{\ln(1/\rho)}\), \(b = \frac{\ln\br{2(2+\lambda) m}}{\ln(1/\rho)}\) and $T' = \frac{T}{\eta_0}$.
\end{restatable}
\begin{proof}
First note that, 
\begin{align*}
\alpha &= \left(\frac{1}{T}\right)^{1/T} = \exp\left(-\frac{1}{T} \, \ln(T) \right)
\end{align*}
Using the expression for \(\taum = a\ln(T')+b\) from~\cref{lemma:reg_taum_bound},
\begin{align*}
\alpha^{\taum} &= \exp\left(-\frac{\taum}{T} \, \ln(T) \right) = \exp\left(-\frac{a \ln(T') + b}{T} \, \ln(T) \right)    
\end{align*}
Define $k := \frac{\ln(T)}{T}$ and $j := [a \ln(T') + b] \, \frac{\ln(T)}{T} = [a \ln(T') + b] \, k$. Using the above relations, 
\begin{align*}
\frac{\alpha^{-\taum}-1}{1-\alpha} &= \frac{\exp(j) - 1}{1 - \exp(-k)} 
\end{align*}
In order to simplify the above expression, we use the following inequalities:
\begin{align*}
\forall x \geq 0 \,, 1 - \exp(-x) \geq \frac{x}{1+x} \quad \text{;} \quad 
\forall y \in (0,1) \,, \exp(2y) \leq \frac{1+y}{1-y} 
\end{align*}
Since $T \geq 1$, $k \geq 0$ and hence we can use $x = k$ to conclude that $1 - \exp(-k) \geq \frac{k}{1+k}$. \\
We substitute \(y\) with \(j/2\), which requires \(0<j<2\). \(j = [a \ln(T') + b] \, \frac{\ln(T)}{T} = \taum\frac{\ln(T)}{T} > 0\) is already satisfied.
Ensuring $j \leq 2$ requires that $[a \ln(T') + b] \, \frac{\ln(T)}{T} \leq 2$. For $\eta_0 \leq 1$ and $T \geq 1$, $T' \geq 1$. Moreover, for $T \geq \frac{1}{\eta_0}$, $\frac{\ln(T')}{\ln(T)} \leq 2$. Hence, it suffices to ensure that
\begin{align*}
2 \, a \, \frac{\ln^2(T)}{T} + b \, \frac{\ln(T)}{T} \leq 2   
\end{align*}
Therefore, it suffices to ensure that 
\begin{align*}
\frac{\ln^2(T)}{T} \leq \frac{1}{2a} \quad \text{and} \quad \frac{\ln(T)}{T} \leq \frac{1}{b}   
\end{align*}
With these constraints on $T$, we can guarantee that $j \leq 1$. Using $y = \frac{j}{2}$ in the above inequality, we can conclude that, 
\begin{align*}
\exp(j) - 1 &\leq \frac{1 + j/2}{1 - j/2} - 1 = \frac{j}{1 - j/2} \leq 2 \, j    
\end{align*}
Combining the above relations, we get that, 
\begin{align*}
\frac{\alpha^{-\taum}-1}{1-\alpha} & \leq \frac{2 \, j \, (k+1)}{k} \\
\intertext{Following the same steps as in the proof of~\cref{lemma:Geo_alpha_bound}, we conclude that, for $T \geq 3$ and $\eta_0 \leq 1$,}
\frac{\alpha^{-\taum}-1}{1-\alpha} & \leq 8 \, \max\{a,b\} \, \ln(T')
\end{align*}
\end{proof}
\RegBBound*
\begin{proof}
    \begin{align*}
        \norm{w_{t+1} - w_r^*} 
        &\le \norm{w_t - w_r^*} + \eta_t \norm{g^r_t(w_t)} \\
        &\le \br{1+(2+\lambda)\eta_t}\norm{w_t - w_r^*} + (3+\lambda)\eta_t \zeta. \tag{by \cref{lemma:reg_mitra7}}
    \end{align*}
    Iterating the above inequality, we have \(\forall   t \le \taum\):
    \begin{align*}
        \norm{w_t - w_r^*} 
        &\le \norm{w_1 - w_r^*}\prod_{i=1}^{t}\br{1+(2+\lambda)\eta_i} 
        + (3+\lambda)\zeta \sum_{i=1}^{t}\eta_i\prod_{j=i+1}^{t}\br{1+(2+\lambda)\eta_j}  \\
        &\le \norm{w_1 - w_r^*}\exp\br{(2+\lambda)\sum_{i=1}^{t}\eta_i}
        + (3+\lambda)\zeta \sum_{i=1}^{t}\eta_i\exp\br{(2+\lambda)\sum_{j=i+1}^{t}\eta_j} \tag{since $1+x \leq \exp(x)$}\\
        &= \norm{w_1 - w_r^*}\exp\br{(2+\lambda)\eta_0\sum_{i=1}^{t}\alpha^i}
        + (3+\lambda)\zeta \eta_0 \sum_{i=1}^{t}\alpha^i\exp\br{(2+\lambda)\eta_0\sum_{j=i+1}^{t}\alpha^j} \\
        &= \norm{w_1 - w_r^*}\exp\br{(2+\lambda)\eta_0\sum_{i=1}^{t}\alpha^i}
        + (3+\lambda)\zeta \eta_0 \sum_{i=1}^{t}\alpha^i\exp\br{(2+\lambda)\eta_0\frac{\alpha^{i+1} - \alpha^{t+1}}{1-\alpha}} \\
        &= \norm{w_1 - w_r^*}\exp\br{(2+\lambda)\eta_0 \frac{\alpha - \alpha^{t+1}}{1 - \alpha}}
        + (3+\lambda)\zeta \eta_0 \sum_{i=1}^{t}\alpha^i\exp\br{(2+\lambda)\eta_0\frac{\alpha^{i+1} - \alpha^{t+1}}{1-\alpha}} \\
        & \leq \norm{w_1 - w_r^*}\exp\br{(2+\lambda)\eta_0 \frac{\alpha - \alpha^{t+1}}{1 - \alpha}}
        + (3+\lambda)\zeta \eta_0 \sum_{i=1}^{t}\alpha^i\exp\br{(2+\lambda)\eta_0\frac{\alpha - \alpha^{t+1}}{1-\alpha}} \tag{since $\alpha^{i+1} \leq \alpha$} \\
        & = \exp\br{(2+\lambda)\eta_0 \frac{\alpha - \alpha^{t+1}}{1 - \alpha}} \left[\norm{w_1 - w_r^*} + (3+\lambda)\zeta \eta_0 \sum_{i=1}^{t}\alpha^i  \right] \\   
        & = \exp\br{(2+\lambda)\eta_0   \alpha \frac{1 - \alpha^{t}}{1 - \alpha}}   \left[\norm{w_1 - w_r^*} + (3+\lambda)\zeta \eta_0   \alpha   \frac{1 - \alpha^{t}}{1- \alpha}\right] \\        
        & \leq \exp\br{4(2+\lambda)\eta_0   \alpha   \max\{a,b\}   \ln(T')}   \left[\norm{w_1 - w_r^*} + 4(3+\lambda) \zeta \eta_0   \alpha   \max\{a,b\}   \ln(T')\right] \tag{by \cref{lemma:Geo_alpha_bound}}\\
        \intertext{since $\eta_0 \leq \frac{1-\gamma}{16 \ln (T)}$,}        
        & \leq \exp\br{\frac{1}{4}(2+\lambda)\alpha(1-\gamma)\max\{a,b\}   \frac{\ln(T')}{\ln(T)}}    \left[\norm{w_1 - w_r^*} + \frac{1}{4}(3+\lambda) \alpha(1-\gamma)\max\{a,b\}   \frac{\ln(T')}{\ln(T)}   \zeta \right] \\
        & \leq \exp\br{\frac{1}{2}(2+\lambda)\alpha(1-\gamma)\max\{a,b\} }    \left[\norm{w_1 - w_r^*} + \frac{1}{2}(3+\lambda) \alpha(1-\gamma)\max\{a,b\}   \zeta \right] \tag{for $T \geq \frac{1}{\eta_0}$, $\frac{\ln(T')}{\ln(T)} \leq 2$}\\
        & \leq \exp\br{\frac{1}{2}(2+\lambda)\max\{a,b\}} \exp\br{\frac{1}{4}(3+\lambda)\max\{a,b\}}   \left[\norm{w_1 - w_r^*} + \zeta \right] \tag{\(\alpha \le 1, (1-\gamma) \le 1, \max\{a, b\} \le \exp\br{\frac{1}{2}\max\{a, b\}}\)} \\
        & \le \exp\br{\left(2+\lambda\right)\max\{a,b\}}   \left[\norm{w_1 - w_r^*} + \zeta \right].
    \end{align*}
    Squaring the both sides, we get:
    \begin{align*}
        \norm{w_t - w_r^*}^2 
        &\le \exp\br{2\left(2+\lambda\right)\max\{a,b\}}   \left[\norm{w_1 - w_r^*} + \zeta \right]^2.
    \end{align*}
    Let \(B(\taum) := \exp\br{2\left(2+\lambda\right)\max\{a,b\}}   \left[\norm{w_1 - w_r^*} + \zeta \right]^2\), we have that \(\norm{w_t - w_r^*}^2 \leq B(\taum)\) for all \(t \leq \taum\).
\end{proof}
Using the above lemmas, we will follow a proof similar to that in~\citet{mitra2025a}. For this, we define the following notation:
\[d_t := \E{}{\norm{w_t - w_r^*}}^2,\]
and 
\[e_t := \E{}{\langle w_t - w_r^*, g^r_t(w_t) - g^r(w_t) \rangle},\]
which includes the error introduced by sampling along the Markov chain. The expectations are taken with respect to state distribution at time \(t\). For the subsequent lemmas, we omit the subscript \(P_\pi^t \mu_0\) for brevity.

\RegTauBound*
\begin{proof}
\begin{align*}
\norm{w_{t} - w_{t-\taum}} &\le \sum_{i=t-\taum}^{t-1} \norm{w_{i+1} - w_i} 
\tag{triangle inequality} \\
&\le \sum_{i=t-\taum}^{t-1} \eta_i \norm{g^r_i(w_i)} \tag{by the update}\\
&\le \sum_{i=t-\taum}^{t-1} \eta_i \br{(2+\lambda)\norm{w_i - w_r^*} + (3+\lambda)\zeta} \tag{by \cref{lemma:reg_mitra7}} \\
& \leq \sum_{i=t-\taum}^{t-1} \eta_i \br{(2+\lambda) \sqrt{B(\taum)} + (3+\lambda)\zeta} \tag{assuming that \(\norm{w_k - w_r^*}^2 \le B(\taum)\) for \(k \in [t]\)} \\
&= \underbrace{\br{(2+\lambda) \sqrt{B(\taum)} + (3+\lambda)\zeta}}_{:= C} \, \eta_0 \, \sum_{i=t-\taum}^{t-1} \alpha^i \tag{by definition of the exponential step-sizes} \\
& = C \, \eta_0 \, \alpha^{t - \taum} \, \sum_{i=0}^{\taum-1} \alpha^i \\
& = C \, \eta_0 \, \frac{\alpha^{t}}{\alpha^{\taum}} \, \frac{1 - \alpha^{\taum}}{1 - \alpha} = C \, \eta_t \, \frac{\alpha^{-\taum} - 1}{1 - \alpha} \\ 
& \leq 8 \, C \, \eta_t \, \max\{a,b\} \, \ln(T') \tag{using~\cref{lemma:Geo_alpha_bound-modified}} \\
\implies \normsq{w_{t} - w_{t-\taum}} & \leq 64 \, C^2 \, \eta^2_t \, [\max\{a,b\}]^2 \, \ln^2(T') \\
& = 64 \, \br{(2+\lambda) \sqrt{B(\taum)} + (3+\lambda)\zeta}^2 \, \eta_t^2 \, [\max\{a,b\}]^2 \, \ln^2(T') \\
& \leq 256 \, \br{(2+\lambda) \sqrt{B(\taum)} + (3+\lambda)\zeta}^2 \, \eta_t^2 \, [\max\{a,b\}]^2 \, \ln^2(T) \tag{for $\eta_0 \leq 1$ and $T \geq 1$, $\frac{\ln(T')}{\ln(T)} \leq 2$} \\
& \leq 256 \, \left[2 \, (2+\lambda)^2 \, B(\taum) + 2 \, (3+\lambda)^2 \, \zeta^2\right] \, \eta_t^2 \, [\max\{a,b\}]^2 \, \ln^2(T) \tag{since $(x+y)^2 \leq 2 x^2 + 2 y^2$} \\
& \leq 256 \, \left[2 \, (2+\lambda)^2 \, B(\taum) + 2 \, (3+\lambda)^2 \, \zeta^2\right] \, \eta_t^2 \, \ln^4(T) \tag{since $\ln(T) \geq \max\{a,b\}$} \\
& \leq \underbrace{2560 \, (2+\lambda)^2}_{:= c_1^2} \, B(\taum) \, \eta_t^2 \, \ln^4(T) \tag{since $\zeta^2 \leq B(\taum)$ and $2 \, (3+\lambda)^2 \leq 8 \, (2+\lambda)^2$} \\
& = c_1^2 \, B(\taum) \, \eta_t^2 \, \ln^4(T)
\end{align*}    
\end{proof}

\RegEBound*
\begin{proof}
    Following \citet{mitra2025a}, we decompose as: \(\langle w_t -w_r^*, g^r_t(w_t) - g^r(w_t)\rangle=T_1+T_2+T_3+T_4\), where
    \begin{align*}
    T_1&=\langle w_t-w_{t-\taum}, g^r_t(w_t)-g^r(w_t)\rangle,\\
    T_2&=\langle w_{t-\taum}-w_r^*, g^r_t(w_{t-\taum})-g^r(w_{t-\taum})\rangle,\\
    T_3&=\langle w_{t-\taum}-w_r^*, g^r_t(w_t)- g^r_t(w_{t-\taum})\rangle, \hspace{2mm} \textrm{and}\\
    T_4&=\langle w_{t-\taum}-w_r^*,g^r(w_{t-\taum})-g^r(w_t)\rangle.\\
    \end{align*}
    For \(T_1\):
    \begin{align*}
    T_1 
    &\le \norm{w_t-w_{t-\taum}}\norm{g^r_t(w_t)-g^r(w_t)} \tag{by Cauchy-Schwarz} \\
    & \leq \frac{1}{2 \, c_1 \, \eta_t \, \ln^2(T)} \normsq{w_t - w_{t-\taum}} + \frac{c_1 \, \eta_t \, \ln^2(T)}{2} \, \normsq{g^r_t(w_t)-g^r(w_t)} \tag{by Young's inequality} \\
    & \leq \frac{c_1 \, \eta_t \, \ln^2(T) \, B(\taum)}{2} + \frac{c_1 \, \eta_t \, \ln^2(T)}{2} \, \normsq{g^r_t(w_t)-g^r(w_t)} \tag{using~\cref{lemma:reg_tau_bound}}    
    \end{align*}    
    Simplifying $\normsq{g^r_t(w_t)-g^r(w_t)}$,
    \begin{align*}
    \normsq{g^r_t(w_t)-g^r(w_t)} & \leq 2\normsq{g^r_t(w_t)} + 2\normsq{g^r(w_t)} \tag{since $(x+y)^2 \leq 2x^2 + 2y^2$}\\
    & \leq 2 \, [\br{2+\lambda}\norm{w_t - w_r^*} + \br{3+\lambda}\zeta]^2 + 2 \, [\br{2+\lambda}\norm{w_t - w_r^*}]^2 \tag{using~\cref{lemma:reg_mitra7} and \cref{lemma:reg_mitra8}} \\
    & \leq 4 \, \br{2+\lambda}^2 \, \normsq{w_t - w_r^*} + 4 \, \br{3 + \lambda}^2 \, \zeta^2 + 2 \, \br{2+\lambda}^2 \, \normsq{w_t - w_r^*} \tag{since $(x+y)^2 \leq 2x^2 + 2y^2$}\\
    & \leq 4 \, \br{2+\lambda}^2 \, B(\taum) + 4 \, \br{3 + \lambda}^2 \, B(\taum) + 2 \, \br{2+\lambda}^2 \, B(\taum) \tag{since \(d_k \le B(\taum)\) for all \(k \in [t]\) and \(\zeta^2 \le B(\taum)\)} \\
    & = \underbrace{(4 \, \br{2+\lambda}^2 + 4 \, \br{3 + \lambda}^2 + 2 \, \br{2+\lambda}^2)}_{:= c_2} \, B(\taum) \\
    \implies \normsq{g^r_t(w_t)-g^r(w_t)} &\leq c_2 \, B(\taum)
    \end{align*}
    Combining the above inequalities, 
    \begin{align*}
    T_1 & \leq \frac{c_1 \, \eta_t \, \ln^2(T) \, B(\taum)}{2} + \frac{c_1 \, c_2 \, \eta_t \, \ln^2(T)}{2} \, B(\taum) = \eta_t \, \ln^2(T) \, B(\taum) \, \underbrace{\left[\frac{c_1}{2} + \frac{c_1 \, c_2}{2} \right]}_{:= C_1} \\
    \implies T_1 & \leq C_1 \, \eta_t \, \ln^2(T) \, B(\taum)
    \end{align*}
    For \(T_3\):
    \begin{align*}
    T_3
    &\le \norm{w_{t-\taum}-w_r^*}\norm{g^r_t(w_t)-g^r_t(w_{t-\taum})} \tag{by Cauchy-Schwarz} \\
    &= \norm{w_{t-\taum}-w_r^*}\norm{g_t(w_t)-g_t(w_{t-\taum}) - \lambda(w_t - w_{t-\taum})} \tag{by definition} \\
    &\le \norm{w_{t-\taum}-w_r^*} \br{\norm{g_t(w_t)-g_t(w_{t-\taum})} + \lambda\norm{w_t - w_{t-\taum}}} 
    \tag{triangle inequality}\\
    &\le \norm{w_{t-\taum}-w_r^*} \br{2\norm{w_{t}-w_{t-\taum}} + \lambda\norm{w_t - w_{t-\taum}}} 
    \tag{by \cref{eq:L_i_lipschitz}} \\
    & = \norm{w_{t-\taum}-w_r^*} \; (2+\lambda) \, \norm{w_{t}-w_{t-\taum}} \\ 
    & \leq \frac{1}{2 \, c_1 \, \eta_t \, \ln^2(T)} \normsq{w_t - w_{t-\taum}} + \frac{c_1 \, \eta_t \, (2+\lambda) \, \ln^2(T)}{2} \, \norm{w_{t-\taum}-w_r^*}^2 
    \tag{by Young's inequality}\\
    & \leq \frac{c_1 \, \eta_t \, \ln^2(T) \, B(\taum)}{2} + \frac{c_1 \, \eta_t \, (2+\lambda) \, \ln^2(T)}{2} \, \norm{w_{t-\taum}-w_r^*}^2 \tag{by~\cref{lemma:reg_tau_bound}} \\
    & \leq \frac{c_1 \, \eta_t \, \ln^2(T) \, B(\taum)}{2} + \frac{c_1 \, \eta_t \, (2+\lambda) \, \ln^2(T) \, B(\taum) }{2} \tag{since \(d_k \le B(\taum)\) for all \(k \in [t]\)} \\
    & = \eta_t \, \ln^2(T) \, B(\taum) \, \underbrace{\left[\frac{c_1}{2} + \frac{c_1 \, (2+\lambda)}{2}\right]}_{:= C_2} \\
    \implies T_3 & \leq C_2 \, \eta_t \, \ln^2(T) \, B(\taum)
    \end{align*}
    For \(T_4\), the same analysis applies.
    \begin{align*}
        T_4
        &\le \norm{w_{t-\taum}-w_r^*}\norm{g^r(w_{t-\taum})-g^r(w_t)} \tag{by Cauchy-Schwarz} \\
        &= \norm{w_{t-\taum}-w_r^*}\norm{g(w_t)-g(w_{t-\taum}) - \lambda(w_t - w_{t-\taum})} \tag{by definition} \\
        &\le \norm{w_{t-\taum}-w_r^*} \br{2\norm{w_{t}-w_{t-\taum}} + \lambda\norm{w_t - w_{t-\taum}}} \tag{by \cref{eq:L_lipschitz}} \\
        &= \norm{w_{t-\taum}-w_r^*} \; (2+\lambda) \, \norm{w_{t}-w_{t-\taum}} \\ 
        \intertext{Following the same analysis for $T_3$, we get that,}
        T_4 & \leq C_2 \, \eta_t \, \ln^2(T) \, B(\taum)
    \end{align*}
    
    For \(T_2\), following the proof in~\citet{mitra2025a}:
    \begin{align*}
    \E{}{T_{2}} &= \E{}{\langle w_{t-\taum} -w_r^*, g^r_t(w_{t-\taum})-g^r(w_{t-\taum})\rangle}\\
    &=\mathbb{E}\left[\mathbb{E}\left[\langle w_{t-\taum} -w_r^*, g^r_t(w_{t-\taum})-g^r(w_{t-\taum})\rangle | w_{t-\taum}\right]\right]\\
    &=\mathbb{E}\left[\langle w_{t-\taum} -w_r^*, \mathbb{E}\left[ g^r_t(w_{t-\taum})-g^r(w_{t-\taum})| w_{t-\taum}\right]\rangle\right]\\
    &\leq \mathbb{E}\sbr{\norm{w_{t-\taum} -w_r^*} \norm{\mathbb{E}\sbr{ g^r_t(w_{t-\taum})-g^r(w_{t-\taum}) \mid w_{t-\taum}}}} \tag{by Cauchy-Schwarz}\\
    &\leq \eta_T \mathbb{E}\left[\Vert w_{t-\taum} -w_r^*\Vert  \left(1+\Vert w_{t-\taum}\Vert\right)\right] \tag{by \cref{eq:reg_taum}}\\
    &\leq \eta_t \mathbb{E}\left[\Vert w_{t-\taum} -w_r^*\Vert  \left(1+\Vert w_{t-\taum}\Vert\right)\right] \tag{exponential step-size decreases}\\
    &\leq \eta_t \mathbb{E}\left[\Vert w_{t-\taum} -w_r^*\Vert \left(1+\Vert w_r^* \Vert + \Vert w_{t-\taum}-w_r^*\Vert\right)\right] \tag{triangle inequality}\\
    &\leq \eta_t \mathbb{E}\left[\Vert w_{t-\taum} -w_r^*\Vert \left(2{\zeta}+\Vert w_{t-\taum} -w_r^* \Vert\right)\right] \tag{by the definition of \(\zeta\)}\\
    &\leq 3\eta_t B(\taum) \tag{assuming \(d_k \le B(\taum)\) for all \(k \in [t]\), \(\norm{w_{t-\taum} - w_r^*} \le B(\taum)\), and \(\zeta \le B(\taum)\)} \\
    \implies \E{}{T_{2}} & \leq 3 \, \eta_t \, \ln^2(T) \, B(\taum) 
    \end{align*}
    
    Combining \(T_1, T_2, T_3, T_4\) yields:
    \begin{align*}
    e_t & \leq \underbrace{(C_1 + 3 + 2 C_2)}_{:= C} \, \eta_t \, \ln^2(T) \, B(\taum) 
    \end{align*}

\end{proof}

In addition to \(e_t\), we also need to upper bound \(\E{s_t\sim P_\pi^t \mu_0}{\norm{g^r_t(w_t) - g^r(w_t)}^2}\).
\RegGVar*
\begin{proof}
    \begin{align*}
        &\E{s_t\sim P_\pi^t \mu_0}{\norm{g^r_t(w_t) - g^r(w_t)}^2} \\
        \le& \E{}{2\norm{g^r_t(w_t)}^2 + 2\norm{g^r(w_t)}^2}. \tag{\(\norm{x-y}^2 \le 2\norm{x}^2 + 2\norm{y}^2\)} \\
        \le& 2\E{}{\br{(2+\lambda)\norm{w_t - w_r^*} + (3+\lambda)\zeta}^2 + \br{(2+\lambda)\norm{w_t - w_r^*}}^2}
            \tag{by \cref{lemma:reg_mitra7} and \cref{lemma:reg_mitra8}} \\
        \le& \underbrace{10(3+\lambda)^2}_{:=C'} \,  B(\taum). 
        \tag{assuming \(d_k \le B(\taum)\) for all \(k \in [t]\)} 
    \end{align*}

\end{proof}

For the subsequent steps, the proofs for standard TD(0) and regularized TD(0) differ. We first provide the convergence rate for standard TD(0), where the step-size depends on \(\omega\), and then show that regularized TD(0) removes this requirement.

\subsection{Standard TD(0)}
\InductionBound*
\begin{proof}
We use the above lemmas and prove the result by induction. We assume that for any \(t \ge \taum\), \(d_k \le B(\taum)\) for all \(k \in [t]\). Now we show that, with an appropriate choice of $\eta_0$, we have \(d_{k+1} \le B\). We continue the expansion in \cref{eq:reg_separate_markovian_noise} and take expectation with respect to the randomness at iteration $t$:
\begin{align*}
    &\E{}{\norm{w_{t+1} - w^*}^2}\\
    & \le \norm{w_t - w^*}^2 + 2\eta_t \E{}{\langle g_t(w_t) - g(w_t), w_t - w^* \rangle} + 2\eta_t^2 \E{}{\norm{g_t(w_t) - g(w_t)}^2} + 2\eta_t^2 \norm{g(w_t)}^2 + 2\eta_t \langle g(w_t), w_t - w^* \rangle \\
    &\le \norm{w_t - w^*}^2 + 2\eta_t \langle g(w_t), w_t - w^* \rangle + 2\eta_t^2 \norm{g(w_t)}^2
    + 2C \, \eta_t^2 \, \ln^2(T) \, B(\taum) + 2C' \, \eta_t^2 \, B(\taum)
    \tag{using~\cref{lemma:reg_e_bound,lemma:reg_g_var}} \\
    & = \norm{w_t - w^*}^2 + 2\eta_t \langle g(w_t), w_t - w^* \rangle + 2\eta_t^2 \norm{g(w_t)}^2 + 2 \, [C \, \ln^2(T) + C'] \, \eta_t^2 \, B(\taum)
\end{align*}
We note that \(\norm{w_t - w^*}^2 + 2\eta_t \langle g(w_t), w_t - w^* \rangle + 2\eta_t^2 \norm{g(w_t)}^2\) is similar to the analysis for the mean-path update in \cref{app:constant_mean}. We continue the analysis as follows:
\begin{align*}
\E{}{\norm{w_{t+1} - w^*}^2}
&\le \norm{w_t - w^*}^2 + \br{16\eta_t^2 - 2(1-\gamma)\eta_t}\norm{V_{w_t} - V_{w^*}}^2 + 2 \, [C \, \ln^2(T) + C'] \, \eta_t^2 \, B(\taum) \tag{by \cref{lemma:bound_prod_with_V} and \cref{lemma:bound_g_with_V}}
\end{align*}
Setting \(\eta_0 \le \frac{1-\gamma}{16 \, \ln(T)} < 1\), we can guarantee \(\eta_t \le \frac{1-\gamma}{16}\), thus
\begin{align*}
\E{}{\norm{w_{t+1} - w^*}^2} & \le \norm{w_t - w^*}^2 - (1-\gamma)\eta_t\norm{V_{w_t} - V_{w^*}}^2
+ 2 \, [C \, \ln^2(T) + C'] \, \eta_t^2 \, B(\taum) \\
& \le \norm{w_t - w^*}^2 - (1-\gamma)\eta_t\omega\norm{w_t - w^*}^2
+ 2 \, [C \, \ln^2(T) + C'] \, \eta_t^2 \, B(\taum)\tag{by \cref{lemma:V_and_w}} 
\end{align*} 
Under the assumption \(d_k \le B(\taum)\) for all \(k \in [t]\), and since $(1 - \gamma)   \eta_t   \omega \leq 1$, we have
\begin{align*}
    \E{}{\norm{w_{t+1} - w^*}^2} \le 
    \br{1 - (1-\gamma) \eta_t \omega + 2 \, [C \, \ln^2(T) + C'] \, \eta_t^2 } \, B(\taum).
\end{align*}
When \(\eta_0 \le \min \left\{\frac{(1-\gamma)\omega}{ 2 \, [C \, \ln^2(T) + C']}, \frac{1 - \gamma}{16 \, \ln(T)} \right\}\), we have \(\br{1 - (1-\gamma) \eta_t \omega + 2 \, [C \, \ln^2(T) + C'] \, \eta_t^2} \le 1\). 

Furthermore, since \(0 < (1-\gamma) \le 1, 0 < \omega < 1, T \ge 3\), we know that \(\eta_0 \leq 1\), and consequently, \((1-\gamma)\eta_t\omega \le (1-\gamma)\eta_0\omega < 1\), implying that \(1-(1-\gamma)\eta_t\omega > 0\). 

Hence, $\br{1 - (1-\gamma) \eta_t \omega + 2 \, [C \, \ln^2(T) + C'] \, \eta_t^2 } \in (0,1)$, and consequently, 
\begin{align*}
    \E{}{\norm{w_{t+1} - w^*}^2} \le B(\taum),
\end{align*}
which completes the induction.

Plugging in the value of \(C\) and \(C'\) with \(\lambda = 0\), we have \(\frac{(1-\gamma)\omega}{ 2  [C  \ln^2(T) + C']} = \frac{(1-\gamma)\omega}{446\ln^2(T) + 180}\). Since \(\omega \le 1\), \(\ln^2(T) \ge \ln(T)\), we have \(\frac{(1-\gamma)\omega}{446\ln^2(T) + 180} \le \frac{1 - \gamma}{16  \ln(T)}\). Thus it suffices to have \(\eta_0 \le \frac{(1-\gamma)\omega}{ 2 \, [C \, \ln^2(T) + C']}\).
\end{proof}

The next theorem quantifies the convergence rate of standard TD(0) with exponential step-sizes under Markovian sampling.
\ExpMarkovThm*
\begin{proof}
Continuing the one-step expansion with step-size \(\eta_0 \le \frac{1-\gamma}{16}\) as in \cref{lemma:induction_bound}, we have that, for the absolute constants $C$ and $C'$ defined in~\cref{lemma:reg_e_bound} and~\cref{lemma:reg_g_var} respectively, 
\begin{align*}
    \E{}{\norm{w_{t+1} - w^*}^2} & \leq \norm{w_t - w^*}^2 - (1-\gamma)\eta_t\omega\norm{w_t - w^*}^2 + \underbrace{2 \, [C \, \ln^2(T) + C']}_{:= C(T)} \, \eta_t^2 \, B(\taum)     
\end{align*} 
Taking expectation over \(t\in [T]\), we have:
\begin{align*}
\E{}{\norm{w_{T} - w^*}^2} & \le  \norm{w_0 - w^*}^2\exp\br{-\eta_0\omega(1-\gamma)\sum_{t=1}^T \alpha^t} \\
&\quad + C(T) \, B(\taum) \eta_0^2\sum_{t=1}^T \alpha^{2t}\exp \br{-\eta_0\omega(1-\gamma)\sum_{i=t+1}^T\alpha^i}.
\end{align*}
This result has the same form as in \cref{sec:exp_iid}. Applying \cref{lemma:X-bound} and \cref{lemma:Y-bound}, we obtain the convergence rate:
\begin{align*}
\E{}{\norm{w_{T+1} - w^*}^2}
&\le \norm{w_0 - w^*}^2e\exp\br{-\eta_0\omega(1-\gamma)\frac{\alpha T}{\ln(T)}} + \frac{8 C(T) \, B(\taum)}{e\br{\omega(1-\gamma)}^2}\frac{\ln^2(T)}{\alpha^2 T} \\
&= O\br{\exp\br{-\frac{\omega^2 \, T}{\ln^3(T)}} + \frac{\ln^4(T)}{\omega^2 T}\exp\br{\frac{m}{\ln(1/\rho)}}} \tag{plugging in the values of $\eta_0, B(\taum), C(T)$},
\end{align*}
where \(m\) and \(\rho\) are related to mixing time as \(\taum = \frac{\ln(4Tm/\eta_0)}{\ln(1/\rho)}\).

Additionally, for the condition \(T \ge \max\{\frac{1}{\eta_0}, 3\}\), when \(\eta_0 \le\frac{(1-\gamma)\omega}{ 2 \, [C \, \ln^2(T) + C']}\), \(T \ge 1/\eta_0\) implies \(T \ge 3\). Thus, it suffices that \(T \ge \frac{1}{\eta_0}\).
\end{proof}

\subsection{Regularized TD(0)}
Now we provide the proof for regularized TD(0), and demonstrate that it does not require \(\omega\).
\RegInductionBound*
\begin{proof}
We use the above lemmas and prove the result by induction. We assume that for any \(t \ge \taum\), \(d_k \le B(\taum)\) for all \(k \in [t]\). Now we show that, with an appropriate choice of $\eta_0$, we have \(d_{k+1} \le B\). We continue the expansion in \cref{eq:reg_separate_markovian_noise} and take expectation with respect to the randomness at iteration $t$:
\begin{align*}
    & \E{}{\norm{w_{t+1} - w^*}^2}\\
    & \leq \norm{w_t - w^*}^2 + 2\eta_t \E{}{\langle g_t(w_t) - g(w_t), w_t - w^* \rangle} + 2\eta_t^2 \E{}{\norm{g_t(w_t) - g(w_t)}^2} + 2\eta_t^2 \norm{g(w_t)}^2 + 2\eta_t \langle g(w_t), w_t - w^* \rangle \\
    & \leq \norm{w_t - w^*}^2 + 2\eta_t \langle g(w_t), w_t - w^* \rangle + 2\eta_t^2 \norm{g(w_t)}^2
    + 2C \, \eta_t^2 \, \ln^2(T) \, B(\taum) + 2C' \, \eta_t^2 \, B(\taum)
    \tag{using~\cref{lemma:reg_e_bound,lemma:reg_g_var}} \\
    & = \norm{w_t - w^*}^2 + 2\eta_t \langle g(w_t), w_t - w^* \rangle + 2\eta_t^2 \norm{g(w_t)}^2 + 2 \, [C \, \ln^2(T) + C'] \, \eta_t^2 \, B(\taum)
\end{align*}
We note that \(\norm{w_t - w^*}^2 + 2\eta_t \langle g(w_t), w_t - w^* \rangle + 2\eta_t^2 \norm{g(w_t)}^2\) is similar to the analysis for a mean-path update. We continue the analysis as follows:
    \begin{align*}
    \E{}{\norm{w_{t+1} - w_r^*}^2} & \leq
    \norm{w_t - w_r^*}^2 + \sbr{2(8+2\lambda^2)\eta_t^2 - 2\lambda\eta_t} \norm{w_t - w_r^*}^2 - 2\eta_t(1-\gamma)\omega\norm{w_t - w_r^*}^2 \\
    &\quad + 2 \, [C \, \ln^2(T) + C'] \, \eta_t^2 \, B(\taum) \tag{by \cref{lemma:reg_convex_mean} and \cref{lemma:reg_var_mean}}  \\
    & \leq \left(1 + \sbr{2(8+2\lambda^2)\eta_t^2 - 2\lambda\eta_t}\right) \, \norm{w_t - w_r^*}^2 + 2 \, [C \, \ln^2(T) + C'] \, \eta_t^2 \, B(\taum) \\
    \end{align*}    
    If $\eta_0 < \frac{1}{2\lambda}$, $\eta_t < \frac{1}{2\lambda}$ and consequently, $2(8+2\lambda^2)\eta_t^2 - 2\lambda\eta_t > -1$. Hence, for $\eta_0 \leq \frac{1}{2\lambda}$, $1 + 2(8+2\lambda^2)\eta_t^2 - 2\lambda\eta_t > 0$.

    Under the assumption \(d_k \le B(\taum)\) for all \(k \in [t]\), and consequently,      
    \begin{align*}
    \E{}{\norm{w_{t+1} - w_r^*}^2} & \leq \br{1 + 2(8+2\lambda^2)\eta_t^2 - 2\lambda\eta_t + 2 \, [C \, \ln^2(T) + C'] \, \eta_t^2} B(\taum) 
    \end{align*}
    For $\eta_0 \leq \frac{\lambda}{[C \, \ln^2(T) + C'] + (8 + 2 \lambda^2)}$, $\br{1 + 2(8+2\lambda^2)\eta_t^2 - 2\lambda\eta_t + 2 \, [C \, \ln^2(T) + C'] \, \eta_t^2} \leq 1$. Hence, 
    \begin{align*}
    \E{}{\norm{w_{t+1} - w_r^*}^2} \leq B(\taum)
    \end{align*}
    
    This completes the induction.

\end{proof}

We now state the final convergence rate for regularized TD(0) under Markovian sampling.
\RegMarkovThm*
\begin{proof}
As in the proof of \cref{lemma:reg_induction_bound}, we obtain that if $ \eta_0 \le \min \left\{\frac{1}{2\lambda}, \frac{1 - \gamma}{16 \, \ln(T)}, \frac{\lambda}{[C \, \ln^2(T) + C'] + (8 + 2 \lambda^2)} \right\}$, 
\begin{align*}
\norm{w_{t+1} - w_r^*}^2 &\leq \norm{w_t - w_r^*}^2 + \sbr{2(8+2\lambda^2)\eta_t^2 - 2\lambda\eta_t} \norm{w_t - w_r^*}^2 - 2\eta_t(1-\gamma)\omega\norm{w_t - w_r^*}^2 \\ & + \underbrace{2 \, [C \, \ln^2(T) + C']}_{:= C(T)} \, \eta_t^2 \, B(\taum) 
\end{align*}
Moreover for $C'' = 10$, since $\eta_t \leq \eta_0$, if $\eta_0 \leq \frac{\lambda}{C''} \leq \frac{2 \, \lambda}{2 \, (8 + 2\lambda^2)}$, $2(8+2\lambda^2)\eta_t^2 - 2\lambda\eta_t < 0$. 

Hence, for 
$
\eta_0 = \min \left\{\frac{1}{2\lambda}, \frac{1 - \gamma}{16 \, \ln(T)}, \frac{\lambda}{[C \, \ln^2(T) + C'] + C''}, \frac{\lambda}{C''} \right\}\,,
$    
\begin{align*}
    &\norm{w_{t+1} - w_r^*}^2 \le \br{1 - 2\eta_t(1-\gamma)\omega} \norm{w_t - w_r^*}^2 + C(T) \, \eta_t^2B(\taum)
\end{align*}
Taking expectations over \(t\in [T]\) and recursing,
\begin{align*}
    \E{}{\norm{w_{T+1} - w_r^*}^2} & \leq \norm{w_1 - w_r^*}^2 \prod_{t=1}^T\br{ 1 - 2\eta_0\alpha^t(1-\gamma)\omega } + C(T) \, B(\taum) \eta_0^2 \sum_{t=1}^T \alpha^{2t} \prod_{i=t+1}^{T} \br{ 1 - 2\eta_0\alpha^t(1-\gamma)\omega } \\
    & \le \norm{w_1 - w_r^*}^2 \exp\br{-2\eta_0\omega(1-\gamma) \sum_{t=1}^T \alpha^t}  + C(T) \, B(\taum) \eta_0^2 \sum_{t=1}^T \alpha^{2t} \exp\br{-2\eta_0\omega(1-\gamma) \sum_{i=t+1}^T \alpha^i} \\
\end{align*}
Similar to the proof in \cref{app:exp_iid}, applying \cref{lemma:X-bound} and \cref{lemma:Y-bound} yields:
\begin{align*}
\E{}{\norm{w_{T+1} - w_r^*}^2} 
    & \le \norm{w_1 - w_r^*}^2 e \exp\br{-2\eta_0\omega(1-\gamma) \frac{\alpha T}{\ln(T)}} + C(T) \, B(\taum) \frac{4}{e^2(\omega(1-\gamma))^2} \frac{\ln^2(T)}{\alpha^2 T}.
\end{align*}
Expressing the result in terms of the distance to \(w^*\):
\begin{align*}
& \E{}{\norm{w_{T+1} - w^*}^2} \\
& \le 2\E{}{\norm{w_{T+1} - w_r^*}^2} + 2\norm{w_r^* - w^*}^2 \tag{since $(x+y)^2 \leq 2 x^2 + 2 y^2$} \\
& \le \norm{w_1 - w_r^*}^2 2e \exp\br{-2\eta_0\omega(1-\gamma) \frac{\alpha T}{\ln(T)}} + \frac{8 \, C(T) \, B(\taum) }{e^2(\omega(1-\gamma))^2} \frac{\ln^2(T)}{\alpha^2 T}+ \frac{2\lambda^2\norm{w^*}^2}{(1 - \gamma)^2 \omega^2}. \tag{by \cref{lemma:reg_opt_distance}}
\end{align*}
Setting \(\lambda = \frac{1}{\sqrt{T}} < 1\) gives:
\begin{align*}
\E{}{\norm{w_{T+1} - w^*}^2} & \le \norm{w_1 - w_r^*}^2 2e \exp\br{-2\eta_0\omega(1-\gamma) \frac{\alpha T}{\ln(T)}} + \frac{8 \, C(T) \, B(\taum) }{e^2(\omega(1-\gamma))^2} \frac{\ln^2(T)}{\alpha^2 T} + \frac{2  \norm{w^*}^2}{\br{\omega(1-\gamma)}^2 T} \\
    &= O\br{\exp\br{-\frac{\omega \sqrt{T}}{\ln^3(T)}} + \frac{\ln^4 (T)}{\omega^2T}\exp\br{\frac{m}{\ln(1/\rho)}}}\,, \tag{plugging in the values of $\eta_0, B(\taum), C(T)$}
\end{align*}
where \(m\) and \(\rho\) are related to mixing time as \(\taum = \frac{\ln(4Tm/\eta_0)}{\ln(1/\rho)}\).

Additionally, for the condition on \(T\), when \(\eta_0 \le \frac{\lambda}{[C \, \ln^2(T) + C'] + (8 + 2 \lambda^2)}\), \(T \ge 1/\eta_0\) implies \(T \ge 3\). Thus, it suffices that \(T \ge \max\{\frac{1}{\eta_0}, \frac{\ln\br{4 T m / \eta_0}}{\ln\nicefrac{1}{\rho}}\}\), and that \(T\) is large enough that \(\frac{\ln^2(T)}{T} \leq \frac{1}{2a}\), \(\frac{\ln(T)}{T} \leq \frac{1}{b}\), and \(\ln(T) \ge \max\{a,b\}\).
\end{proof}

\section{Helper Lemmas}
\label{app:helper}
\begin{lemma} \label{lemma:X-bound}
\begin{equation*}
X := \sum_{t=1}^T \alpha^t 
\ge \frac{\alpha T}{\ln(T)} - \frac{1}{\ln(T)}.
\end{equation*}
\end{lemma}

\begin{proof}
\begin{equation*}
\sum_{t=1}^T \alpha^t
= \frac{\alpha - \alpha^{T+1}}{1-\alpha}
= \frac{\alpha}{1-\alpha}
  -
  \frac{\alpha^{T+1}}{1-\alpha}.
\end{equation*}

We have
\begin{equation*}
\frac{\alpha^{T+1}}{1-\alpha}
= \frac{\alpha}{T(1-\alpha)}
= \frac{1}{T}\frac{1}{\nicefrac{1}{\alpha} - 1}
\le \frac{1}{T}\frac{1}{\ln(\nicefrac{1}{\alpha})}
= \frac{1}{\ln(T)},
\end{equation*}
where in the inequality we used Lemma 4 and the fact that \(1/\alpha>1\). Plugging back into \(X\) we get

\begin{equation*}
X 
\ge \frac{\alpha}{1-\alpha}
  -
  \frac{1}{\ln(T)}
\ge
\frac{\alpha}{\ln(1/\alpha)}
  -
  \frac{1}{\ln(T)}
=
\frac{\alpha T}{\ln(T)}
  -
  \frac{1}{\ln(T)}.
\end{equation*}
\end{proof}

\begin{lemma} \label{lemma:Y-bound}
    For \(\alpha = \frac{1}{T}^{\nicefrac{1}{T}}\) and any \(\kappa > 0\),
    \begin{equation*}
    \sum_{t=1}^T \alpha^{2t}
    \exp\left(-a\sum_{i=t+1}^T \alpha^i\right)
    \le
    \frac{4c\left(\ln(T)\right)^2}{a^2e^2\alpha^2T},
    \end{equation*}
\end{lemma}
where \(c = \exp\left( a\frac{1}{\ln(T)} \right)\).
\begin{proof}
First, observe that, 
\begin{align*}
\sum_{i=t+1}^T \alpha^i = \frac{\alpha^{t+1} - \alpha^{T+1}}{1 - \alpha} 
\end{align*}
We have
\begin{align*}
    \frac{\alpha^{T+1}}{1 - \alpha} 
    &= \frac{\alpha}{T(1 - \alpha)} 
    = \frac{1}{T} \cdot \frac{1}{\nicefrac{1}{\alpha} - 1} \leq \frac{1}{T} \cdot \frac{1}{\ln(\nicefrac{1}{\alpha})} 
    = \frac{1}{\ln(T)}
    \end{align*}
These relations imply that, 
\begin{align*}
&\sum_{i=t+1}^T \alpha^i 
\geq 
\frac{\alpha^{t+1}}{1 - \alpha} - \frac{1}{\ln(T)} \\
&\implies 
\exp\left( -a \sum_{i=t+1}^T \alpha^i \right) 
\leq 
\exp\left( -a \frac{\alpha^{t+1}}{1- \alpha} + a \frac{1}{\ln(T)}\right) 
= c \exp\left(-a \frac{\alpha^{t+1}}{1- \alpha} \right),
\end{align*}
where \(c = \exp\left( a\frac{1}{\ln(T)} \right)\).
We then have
\begin{align*}
\sum_{t=1}^T \alpha^{2t} \exp\left( -a \sum_{i=t+1}^T \alpha^i \right) 
& \leq c \sum_{t=1}^T \alpha^{2t} \exp\left(-a \frac{\alpha^{t+1}}{1- \alpha} \right) \\
& \leq c \sum_{t=1}^T \alpha^{2t} \left( \frac{2 (1-\alpha) }{e a \alpha^{t+1}} \right)^2 & \tag{by \cref{lem:ineq2} with \(\nu=2\)} \\
& = \frac{4 c}{a^2 e^2 \alpha^2}   T (1 - \alpha)^2 \\
& \leq \frac{4 c}{a^2 e^2 \alpha^2}   T (\ln(\nicefrac{1}{\alpha}))^2 \\
& = \frac{4 c (\ln(T))^2}{a^2 e^2 \alpha^2 T}
\end{align*}
\end{proof}

\begin{lemma}
\label{lem:ineq2}
For all $x, \nu >0$,
\begin{align*}
    \exp(-x) \leq \left( \frac{\nu}{ex }\right)^\nu
\end{align*}
\end{lemma}
\begin{proof}
Let $x > 0$. Define $f(\nu) = \left( \frac{\nu}{ex }\right)^\nu - \exp(-x)$. We have
\begin{align*}
    f(\nu) = \exp\left( \nu\ln(\nu) - \nu\ln(ex)\right) - \exp(-x)
\end{align*}
and
\begin{align*}
    f'(\nu) = \left( \nu \cdot \frac{1}{\nu} + \ln(\nu) - \ln(ex)\right) \exp\left( \nu\ln(\nu) - \nu\ln(ex)\right)
\end{align*}
Thus
\begin{align*}
    f'(\nu) \geq 0 &\iff 1 + \ln(\nu) - \ln(ex) \geq 0 \iff \nu \geq \exp\left( \ln(ex) - 1\right) = x
\end{align*}
So $f$ is decreasing on $(0, x]$ and increasing on $[x, \infty)$. Moreover,
\begin{align*}
    f(x) = \left( \frac{x}{ex}\right)^x - \exp(-x) = \left(\frac{1}{e}\right)^x - \exp(-x) = 0
\end{align*}
and thus $f(\nu) \geq 0$ for all $\nu > 0$ which proves the lemma.
\end{proof}

\end{document}